\documentclass{article}
\usepackage[margin=1in]{geometry} 
\usepackage{multirow}

\usepackage[utf8]{inputenc} 
\usepackage[T1]{fontenc}    
\usepackage{url}            
\usepackage{booktabs}       
\usepackage{amsfonts}       
\usepackage{nicefrac}       
\usepackage{microtype}      
\usepackage{bm} 
\usepackage{color,soul}
\usepackage{setspace}
\usepackage{float}
\usepackage{subfloat}
\usepackage[utf8]{inputenc} 
\usepackage[T1]{fontenc}    
\usepackage{url}            
\usepackage{booktabs}       
\usepackage{amsfonts}       
\usepackage{nicefrac}       
\usepackage{microtype}      
\date{\vspace{-5ex}}

\usepackage[shortlabels]{enumitem}

\usepackage{mathtools,amssymb,graphicx,comment}
\usepackage{amsmath,amsthm}
\usepackage{bbm}
\usepackage{subcaption}
  {
      
  }
  \usepackage{breqn} 
\usepackage{enumitem}
  \usepackage{tcolorbox}
\usepackage{hyperref}

\definecolor{darkred}{RGB}{150,0,0}
\definecolor{darkgreen}{RGB}{0,100,0}
\definecolor{darkblue}{RGB}{0,0,180}
\hypersetup{colorlinks=true, linkcolor=darkred, citecolor=darkgreen, urlcolor=darkblue}
\newcommand\blfootnote[1]{%
  \begingroup
  \renewcommand\thefootnote{}\footnote{#1}%
  \addtocounter{footnote}{-1}%
  \endgroup
}
\newcommand{\widesim}[2][1.5]{
  \mathrel{\overset{#2}{\scalebox{#1}[1]{$\sim$}}}}

\newcommand{\alphaL}{\alpha_ {_{\Lm,\la}}}
\newcommand{\tauL}{\tau {_{\Lm,\la}}}
\newcommand{\Cclin}{\Cc_{\rm lin}}

\newcommand{\Psit}{\widetilde{\Psi}}

\newcommand{\Phit}{\widetilde{\Phi}}

\newcommand{\lat}{\widetilde{\la}}
\makeatletter
\newcommand*{\rom}[1]{\expandafter\@slowromancap\romannumeral #1@}
\makeatother\newcommand{\thetao}{\overline{\theta}}

\newcommand{\mathleft}{\@fleqntrue\@mathmargin0pt}
\newcommand{\mathcenter}{\@fleqnfalse}
\newcommand{\corr}[2]{{\rm{corr}}\left(\,{#1}\,,\,{#2}\,\right)}
\makeatletter
\newcommand{\ssymbol}[1]{^{\@fnsymbol{#1}}}
\makeatother
\newcommand{\env}[3]{\mathcal{M}_{{#1}}\left({#2};{#3}\right)}
\newcommand{\prox}[3]{\mathrm{prox}_{{#1}}\left({#2};{#3}\right)}
\newcommand{\proxp}[3]{\mathrm{prox'}_{{#1},1}\left({#2};{#3}\right)}

\newcommand{\envdx}[3]{\mathcal{M}^{\prime}_{{#1},1}\left({#2};{#3}\right)}

\newcommand{\envdla}[3]{\mathcal{M}^{\prime}_{{#1},2}\left({#2};{#3}\right)}
\newcommand{\envddx}[3]{\mathcal{M}^{''}_{{#1},1}\left({#2};{#3}\right)}

\newcommand{\ourx}{\al G + \mu S f(S)}

\newcommand{\R}{\mathbb{R}}

\newcommand{\al}{\alpha}

\newcommand{\xh}{\widehat{\x}}

\newcommand{\Lm}{\mathcal{L}}

\newcommand{\rP}{\stackrel{P}{\longrightarrow}}

\DeclarePairedDelimiterX{\inp}[2]{\langle}{\rangle}{#1, #2}

\newcommand{\ksi}{\xi}

\newcommand{\wh}{{\widehat\w}}

\usepackage{dsfont}

\newcommand{\simiid}{\widesim{\text{\small{iid}}}}

\newcommand{\Pro}{\mathbb{P}}

\theoremstyle{theorem}
\newtheorem{propo}{Proposition}[section]
\newtheorem{thm}{Theorem}[section]
\newtheorem{lem}{Lemma}[section]
\newtheorem{cor}{Corollary}[section]
\newtheorem{ass}{Assumption}
\theoremstyle{remark}
\theoremstyle{definition}

\newcommand{\sign}{\mathrm{sign}}

\newcommand{\Exp}{\mathbb{E}}               
\newcommand{\E}{\mathbb{E}}                    
\newcommand{\la}{{\lambda}}                     

\newcommand{\sig}{\sigma}

\newcommand{\nn}{\notag}

\newcommand{\x}{\mathbf{x}}

\newcommand{\w}{\mathbf{w}}

\newcommand{\vb}{\mathbf{v}}

\newcommand{\ab}{\mathbf{a}}

\newcommand{\Nn}{\mathcal{N}}
\newcommand{\Lc}{\mathcal{L}}
\newcommand{\Cc}{\mathcal{C}}

\newcommand{\Ec}{\mathcal{E}}

\newcommand{\Ic}{\mathcal{I}}

\newcommand{\Lmt}{\widetilde{\mathcal{L}}}
\newcommand{\beq}{\begin{equation}}
\newcommand{\eeq}{\end{equation}}
\newcommand{\bea}{\begin{align}}
\newcommand{\eea}{\end{align}}

\newcommand{\vp}{\vspace{4pt}}

\usepackage{float}
\usepackage{amsmath}
\def\bea#1\eea{\begin{align}#1\end{align}}
\usepackage{xcolor}

\title{Fundamental Limits of Ridge\,-Regularized \\Empirical Risk Minimization in High Dimensions}

\author{%
Hossein Taheri, Ramtin Pedarsani, and  Christos Thrampoulidis\\
\blfootnote{All authors are with the Electrical and Computer Engineering Department, University of California, Santa Barbara, Santa Barbara, CA 93106, USA. Emails: \{hossein, ramtin, cthrampo\}@ucsb.edu .
}} 

\begin{document}

\maketitle
\vspace{.2in}

\date{\large \hspace{2.4in} June 16, 2020}
\vspace{.2in}
\begin{abstract}
Empirical Risk Minimization (ERM) algorithms are widely used in a variety of estimation and prediction tasks in signal-processing and machine learning applications. Despite their popularity, a theory that explains their statistical properties in modern regimes where both the number of measurements and the number of unknown parameters is large is only recently emerging. In this paper, we characterize for the first time the fundamental limits on the statistical accuracy of convex ERM for inference in high-dimensional generalized linear models. For a stylized setting with Gaussian features and problem dimensions that grow large at a proportional rate, we start with sharp performance characterizations and then derive tight lower bounds on the estimation and prediction error that hold over a wide class of loss functions and for any value of the regularization parameter. Our precise analysis has several attributes. First, it leads to a recipe for optimally tuning the loss function and the regularization parameter. Second, it allows to precisely quantify the sub-optimality of popular heuristic choices: for instance, we show that optimally-tuned least-squares is (perhaps surprisingly) approximately optimal for standard logistic data, but the sub-optimality gap grows drastically as the signal strength increases. Third, we use the bounds to precisely assess the merits of ridge-regularization as a function of the over-parameterization ratio. Notably, our bounds are expressed in terms of the Fisher Information of random variables that are simple functions of the data distribution, thus making ties to corresponding bounds in classical statistics.
\end{abstract}


\section{Introduction}\label{sec:intro}

Empirical Risk Minimization (ERM) includes a wide family of statistical inference algorithms that are popular in estimation and learning tasks encountered in a range of applications in signal processing, communications and machine learning. ERM methods are often efficient in implementation, but first one needs to make certain choices: such as, choose an appropriate loss function and regularization function, 
and tune the regularization parameter. Classical statistics have complemented the practice of ERM with an elegant theory regarding optimal such choices, as well as, fundamental limits, i.e., tight bounds on their performance, e.g., \cite{huber2011robust}. These classical theories typically assume that the size $m$ of the set of observations is much larger than the dimension $n$ of the parameter to be estimated, i.e., $m\gg n$. In contrast, modern inference problems are typically high-dimensional, i.e. $m$ and $n$ are of the same order and often $n>m$ \cite{candes2014mathematics,montanari2015statistical,karoui2013asymptotic}. This paper studies the fundamental limits of convex ERM  in 
high-dimensions for generalized linear models.

Generalized linear models (GLM) relate the response variable $y_i$ to a linear model $\ab_i^T\x_0$ via a link function: $y_i=\varphi(\ab_i^T\,\x_0)$. Here, $\x_0\in\R^n$ is a vector of true parameters and $\ab_i\in\R^n,~i\in[m]$ are the feature (or, measurement) vectors. Following the ERM principle, $\x_0$ can be estimated by the minimizer of the empirical risk $\frac{1}{m}\sum_{i=1}^m \mathcal{L}\left(y_i,\ab_i^T \x\right)$ for a chosen loss function $\Lm$. Typically, ERM is combined with a regularization term and among all possible choices arguably the most popular one is ridge regularization, which gives rise to ridge-regularized ERM (RERM, in short):
\bea\label{eq:RERM_gen}
\xh_{\Lm,\la}=\arg\min_{\x\in\R^n}\,\frac{1}{m}\,\sum_{i=1}^m \mathcal{L}\left(y_i,\ab_i^T \x\right)+\frac{\lambda}{2}\|\x\|_2^{2}.
\eea
This paper aims to provide answers to the following questions on fundamental limits of \eqref{eq:RERM_gen}: \emph{What is the minimum achievable (estimation/prediction) error of $\xh_{\Lm,\la}$? How does this depend on the link function $\varphi$ and how to choose $\Lm$ and $\la$ to achieve it? What is the sub-optimality gap of popular choices such as ridge-regularized least-squares (RLS)? How do the answers to these questions depend on the over-parameterization ratio $n/m$?} We provide answers to the questions above for the following two popular instances of GLMs. 

\noindent\emph{Linear models:} $y_i = \ab_i^T\x_0 + z_i$,  where $z_i\simiid P_Z,~i\in[m]$. As is typical, for linear models, we measure performance of $\xh_{\Lm,\la}$ with the squared error: $\|\xh_{\Lm,\la}-\x_0\|_2^2$.

\vp
\noindent\emph{Binary models:} $y_i = f(\ab_i^T\x_0),~i\in[m]$ for a (possibly random) link function outputing values $\{\pm 1\}$, e.g., logistic, probit and signed models. We measure estimation performance in terms of  (normalized) correlation ${(\xh_{\Lm,\la}^T\,\x_0)}\Big/{\|\xh_{\Lm,\la}\|_2\|\x_0\|_2}$ and prediction performance in terms of classification error $\Pro\big( y\neq \sign(\xh_{\Lm,\la}^T\,\ab) \big)$ where the probability is over a fresh data point $(\ab,y)$.

\vp
All our results are valid under the following two assumptions.
\begin{ass}[High-dimensional asymptotics]\label{ass:HD}
Throughout the paper, we assume the high-dimensional limit where $m,n\rightarrow\infty$ at a fixed ratio $\delta=m/n>0$. 
\end{ass}
\begin{ass}[Gaussian features]\label{ass:gaussian}
The feature vectors $\mathbf{a}_{i} \in \mathbb{R}^{n}, i \in[m]$ are  iid $\mathcal{N}(\mathbf{0},\mathbf{I}_n)$.
\end{ass}

\vp
\noindent\textbf{Overview of Contributions.}~~We are now ready to summarize the paper's main contributions.

\vp
\noindent$\bullet$~~For linear models, we prove a lower bound on the squared-estimation error of RERM; see Theorem \ref{thm:lowerbound_reg}. We start with a system of two nonlinear equations that is parametrized by the loss $\Lc$ and the regularizer $\la$, and determines the high-dimensional limit of the error for the corresponding $\Lc$ and $\la$ \cite{karoui2013asymptotic,Master}. By identifying an algebraic structure in these equations, we establish a lower bound on their solution that holds for all choices of $\Lc$ and $\la$. 

\vp
\noindent$\bullet$~~For binary models, we first derive a system a of three nonlinear equations whose unique solution characterizes the statistical performance (correlation or classification error) of RERM under mild assumptions on the loss and link functions $\Lc$ and $f$; see Theorem \ref{propo:boundedness}. Previous works have only considered specific loss and link functions or \emph{no} regularization. Second, we use this system of equations to upper bound the accuracy over this class of $(\Lm,f)$-pairs; see Theorem \ref{thm:lowerbound_bin}.

\vp
\noindent$\bullet$~~Importantly, we present a recipe for optimally tuning $\Lm$ and $\la$ in both linear and binary models; see Lemmas \ref{thm:opt_reg} and \ref{thm:opt_bin}. For specific models, such as linear model with additive exponential noise, binary logistic and signed data, we numerically show that the optimal loss function is convex and we use gradient-descent to optimize it. The numerical simulations perfectly match with the theoretical predictions suggesting that our bounds are tight. 

\vp
\noindent$\bullet$~~We derive simple closed-form approximations to the aforementioned bounds; see Corollaries \ref{cor:lowerbound} (linear) and \ref{cor:lowerbound_binary} (binary).  These simple (yet tight) expressions allow us to precisely quantify the sub-optimality of ridge-regularized least-squares (RLS). For instance, we show that \emph{optimally-tuned RLS is (perhaps surprisingly) approximately optimal for logistic data and small signal strength, but the sub-optimality gap grows drastically as signal strength increases.} In the Appendix, we also include comparisons to ERM without regularization and to a simple averaging method.
\subsection{Prior Work} 
Our results fit in the rapidly growing recent literature on \emph{sharp} asymptotics of (possibly non-smooth) convex optimization-based estimators, e.g., \cite{DMM,Sto,montanariLasso,Cha,TroppEdge,oymak2016sharp,StoLasso,OTH13,COLT,karoui2013asymptotic,karoui15,donoho2016high,Master,TroppUniversal,miolane2018distribution,wang2019does,celentano2019fundamental,hu2019asymptotics,bu2019algorithmic,hu2019asymptotics,bu2019algorithmic}.  Most of these works study linear models. Extensions to generalized linear models for the special case of regularized LS were studied in \cite{NIPS}, while more recently there has been a surge of interest in RERM methods tailored to binary models (such as logistic regression or SVM) \cite{huang2017asymptotic,candes2018phase,sur2019modern,mai2019large,logistic_regression,svm_abla,salehi2019impact,taheri2020sharp,Zeyu2019,montanari2019generalization,liang2020precise,mignacco2020role}.

Out of these works relatively few have focused on fundamental limits among families of ERM (rather than specific instances). The papers \cite{bean2013optimal,donoho2016high,advani2016statistical} derive lower bounds and optimal loss functions for the squared error of (unregularized) ERM for linear models. In a related work, \cite{montanari15} studies robustness of these methods to the noise distribution. More recently, \cite{celentano2019fundamental} performed an in-depth analysis of fundamental limits of  convex-regularized LS for linear models of structured signals. For binary models, upper bounds on the correlation of un-regularized ERM were only recently derived in \cite{taheri2020sharp}. This paper contributes to this line work. For linear models, we build on corresponding sharp error characterizations in \cite{karoui15,Master} to extend the results of \cite{bean2013optimal,donoho2016high,advani2016statistical} to ridge-regularized ERM. Specifically, our results hold for all values of $\delta>0$ including the, so called, overparameterized regime $\delta<1$. For binary models, our contribution is twofold: (i) we present sharp asymptotic characterizations for RERM for a wide class of loss and link functions;  (ii) we use these to extend the correlation bounds of \cite{taheri2020sharp} to the regularized case. 

On a technical level, the sharp asymptotics are derived using the convex Gaussian min-max Theorem (CGMT) \cite{StoLasso,COLT}. In particular, we follow the machinery introduced in \cite{NIPS,svm_abla,salehi2019impact,taheri2020sharp,Zeyu2019} that applies the CGMT to binary models and predicts the performance in terms of a system of few nonlinear equations. Our main technical contribution here is proving existence and uniqueness of the solutions to these equations, which is critical as it guarantees that our performance bounds hold for a wide class of loss and link functions.
\vp
\noindent\textbf{Notation.}~We use boldface notation for vectors. We write $i\in[m]$ for $i=1,2,\ldots,m$. 
For a random variable $H$ with density $p_{_H}(h)$ that has a derivative $p_{_H}^{\prime}(h), \forall h \in \mathbb{R},$ we define its \emph{Fisher information}
$\mathcal{I}(H):=\mathbb{E}[\left(p_{_H}'(h)/p_{_H}(h)\right)^{2}].$
We write
$\env{\Lm}{x}{\tau}:=\min_{v}\frac{1}{2\tau}(x-v)^2 + \Lm(v),$
for the \emph{Moreau envelope function} and $\prox{\Lm}{x}{\tau}:=\arg\min_{v}\frac{1}{2\tau}(x-v)^2 + \Lm(v)$ for the \emph{proximal operator} of the loss $\Lm:\R\rightarrow\R$ at $x$ with parameter $\tau>0$.
We denote the first order derivative of the Moreau-envelope function w.r.t $x$ as:
$
\envdx{\Lm}{x}{\tau}:=\frac{\partial{\env{\Lm}{x}{\tau}}}{\partial x}.
$
Finally, for a sequence of random variables $\mathcal{X}_{m,n}$ that converges in probability to some constant $c$ in the high-dimensional asymptotic limit of Assumption \ref{ass:HD}, we write $\mathcal{X}_{m,n}\rP c$.


\section{Linear Models}\label{sec:linear}
Consider data $(y_i,\ab_i)$ from an additive noisy linear model:
$y_i = \ab_i^T\x_0 + z_i,~ z_i\simiid  P_Z,~i\in[m].$

\begin{ass}[Noise distribution]\label{ass:noise}
The noise variables $z_i$ are iid distributed as $Z \sim P_Z$, $i\in[m]$, for a distribution $P_Z$ with zero mean and finite nonzero second moment. 
\end{ass}
For loss functions that are lower semicontinuous (lsc), proper,  and convex we focus on the following version of \eqref{eq:RERM_gen} that is tailored to linear models:
\bea\label{eq:opt_reg_main}
\xh_{{\Lm,\la}}:=\arg \min_{\x\in\R^n} \;\;\frac{1}{m}\sum_{i=1}^m \Lm\left(y_i-\ab_i^T \x\right)+\frac{\la}{2}\|\x\|^{2}.
\eea
We assume without loss of generality that $\|\x_0\|_2=1$ \footnote{Suppose that $\|\x_0\|_2=r>0$. Then,  the optimization problem in \eqref{eq:opt_reg_main} can be transformed to the case  $\widetilde{\x}_0:= \x_0/r$ (hence $\|\widetilde{\x}_0\|=1$) by setting $\widetilde{\Lm}(t) := \Lm(rt)$, $\widetilde{\la} := r^2\la$ and $\widetilde{Z}=Z/r$. This implies that the results of Section \ref{sec:limit_opt_lin} can be reformulated by replacing $Z$ with $\widetilde{Z}$.}.

\subsection{Background on Asymptotic Performance}\label{sec:back_lin}
Prior works have investigated the limit of the squared error $\|\xh_{_{\Lm,\la}}-\x_0\|^2$  \cite{karoui2013asymptotic,Master}. Specifically, consider the following system of two equations in two unknowns $\alpha$ and $\tau$:
\begin{subequations}\label{eq:eq_main0}
\begin{align}
\E \Big[\Big(\envdx{\Lm}{\al\,G+Z}{\tau}\Big)^2\,\Big]&=\frac{\alpha^2- \la^2\delta^2 \tau^2}{\tau^2\,\delta},
\\
\E\Big[G\cdot\envdx{\Lm}{\al\,G+Z}{\tau}\Big]&=\frac{\alpha(1-\la\delta\tau)}{\tau\,\delta},
\end{align}
\end{subequations}
where $G\sim\Nn(0,1)$ and $Z\sim P_Z$ is the noise variable. It has been shown in \cite{karoui2013asymptotic,Master} that under appropriate regularity conditions on $\Lm$ and the noise distribution $P_Z$, the system of equations above has a unique solution $(\alpha_{\Lm,\la}>0,\tauL>0)$ and $\alpha_{\Lm,\la}^2$ is the HD limit of the squared-error, i.e., 
\bea\label{eq:alphaL}
\|\xh_{{\Lm,\la}}-\x_0\|_2^2\,\rP\,\alpha_{\Lm,\la}^2. 
\eea
Here, we derive tight lower bounds on $\alpha_{\Lm,\la}^2$ over both the choice of $\Lm$ and $\la$. Our starting point is the asymptotic characterization in \eqref{eq:alphaL}, i.e., our results hold for all loss functions and regularizer parameters for which \eqref{eq:eq_main0} has a unique solution that characterizes the HD limit of the square-error. To formalize this, we define the following collection of loss functions $\Lm$ and noise distributions $P_Z$:
\bea
\Cclin := \Big\{ (\Lm,P_Z)\,\Big|\,\text{$\forall\la>0$: \eqref{eq:eq_main0} has a unique bounded solution $(\alphaL>0,\tau_{_{\Lm,\la}}>0)$ and \eqref{eq:alphaL} holds} \Big\}.\nn
\eea
We refer the reader to \cite[Thm. 1.1]{karoui2013asymptotic} and \cite[Thm. 2]{Master} for explicit characterizations of $(\Lm,P_Z)$ that belong to $\Cclin$. We conjecture that some of these regularity conditions (e.g., the differentiability requirement) can in fact be relaxed. While this is beyond the scope of this paper, if this is shown then automatically the results of this paper formally hold for a richer class of loss functions.

%
%
%
%

\subsection{Fundamental Limits and Optimal Tuning}\label{sec:limit_opt_lin}


Our first main result, stated as Theorem \ref{thm:lowerbound_reg} below, establishes a tight bound on the achievable values of $\alpha_{{\Lm,\la}}^2$ for all regularization parameters $\la>0$ and all choices of $\Lm$ such that $(\Lm,P_Z)\in\Cclin$.

\begin{thm}[{Lower bound on ${\alpha_{{\Lm,\la}}}$}] \label{thm:lowerbound_reg}
Let Assumptions \ref{ass:HD}, \ref{ass:gaussian} and \ref{ass:noise} hold. For $G\sim \mathcal{N}(0,1)$ and noise random variable $Z\sim P_Z$, consider a new random variable $V_a:=a\,G +Z,$ parameterized by $a\in\R$.
Fix any $\delta>0$ and define $\alpha_{\star} = \alpha_\star(\delta,P_Z)$ as follows:
\bea\label{eq:alphaopt_thm}
\alpha_{\star}:=\min_{\substack{0\le x<1/\delta}}\left[a>0:\;\frac{\delta(a^2-x^2\,\delta^2)\,\Ic(V_a)}{(1-x\,\delta)^2}=1\right].
\eea
For any $\Lm$ such that $(\Lm,P_Z)\in\Cclin$, $\la>0$ and $\alpha_ {{\Lm,\la}}^2$ denoting the respective high-dimensional limit of the squared-error as in \eqref{eq:alphaL}, it holds that $\alpha_ {{\Lm,\la}} \geq \alpha_\star$.
\end{thm}
The proof of the theorem is presented in Section \ref{sec:proof_bin_lowerbound}. This includes showing that the minimization in \eqref{eq:alphaopt_thm} is feasible for any $\delta>0$. In general, the lower bound $\alpha_\star$ can be computed  by numerically solving \eqref{eq:alphaopt_thm}. For special cases of noise distributions (such as Gaussian), it is possible to analytically solve \eqref{eq:alphaopt_thm} and obtain a closed-form formula for $\alpha_\star$, which is easier to interpret. While this is only possible for few special cases, our next result establishes a simple closed-form lower bound on $\alpha_\star$ that is valid under only mild assumptions on $P_Z$.
For convenience, let us define $h_{\delta}:\R_{>0}\rightarrow\R_{>0}$,
\bea\label{eq:h_del}
h_{{\delta}}(x):=\frac{1}{2}\Big(1-x-\delta+\sqrt{(1+\delta+x)^2-4\delta}\,\Big)\,.
\eea
The subscript $\delta$ emphasizes the dependence of the function on the oversampling ratio $\delta$. We also note for future reference that $h_{{\delta}}$ is strictly increasing for all fixed $\delta>0$.

\begin{cor}[Closed-form lower bound on $\alpha_\star^2$]\label{cor:lowerbound}
Let $\alpha_\star$ be as in \eqref{eq:alphaopt_thm} under the assumptions of Theorem \ref{thm:lowerbound_reg}. Assume that $p_Z$ is differentiable and takes strictly positive values on the real line. Then, it holds that 
$$\alpha_\star^2 \ge h_{\delta}\left({1}\big/{\Ic(Z)}\right).$$
Moreover, the equality holds if and only if $Z\sim\mathcal{N}(0,\zeta^2)$ for $\zeta>0$. 
\end{cor}

The proof of Corollary presented in Section \ref{sec:proofofcor_lin} shows that the gap between the actual value of $\alpha_\star$ and $h_{\delta}\big({1}\big/{\Ic(Z)}\big)$ depends solely on the distribution of $Z$. Informally: the more $Z$ resembles a Gaussian, the smaller the gap. 
The simple approximation of Corollary \ref{cor:lowerbound} is key for comparing the performance of optimally tuned RERM  to optimally-tuned RLS in Section \ref{sec:LS_linear}.

A natural question regarding the lower bound of Theorem \ref{thm:lowerbound_reg} is whether it is tight.  Indeed, the lower bound cannot be improved in general. This can be argues as follows. Consider the case of additive Gaussian noise $Z\sim\Nn(0,\zeta^2)$ for which $\Ic(Z)={1}/{\E[Z^2]}=1/\zeta^2$. On the one hand, Corollary \ref{cor:lowerbound} shows that $\alpha_\star^2\geq h_{\delta}(\zeta^{2})$ and on the other hand, we show in Lemma \ref{cor:LS_reg} that optimally-tuned RLS achieves this bound, i.e., ${\alpha_{\,{\ell_{_2},\la_{\rm opt}}}^2}=h_{\delta}(\zeta^{2})$. Thus, the case of Gaussian noise shows that the bound of Theorem \ref{thm:lowerbound_reg} cannot be improved in general. 

\vp
Our next result reinforces the claim that the bound is actually tight for a larger class of noise distributions. 
%
\begin{lem}[{Optimal tuning for linear RERM}]\label{thm:opt_reg}
For given $\delta>0$ and $P_Z$, let $(\alpha_{\star}>0,x_\star\in[0,1/\delta))$ be the optimal solution in the minimization in \eqref{eq:alphaopt_thm}. Denote $\la_\star=x_\star$ and define $V_\star := \alpha_\star G +Z$. Consider the loss function
$\Lm_{\star}:\R\rightarrow\R$ defined as 
$\Lm_{\star}(v) := -\env{\frac{\alpha_{\star}^{^2}-\la_{\star}^{^2}\,\delta^{^2}}{1-\la_{\star}\,\delta}\cdot\log\left(p_{_{V_{\star}}}\right)}{v}{1}.$
Then for $\Lm_\star$ and $\la_\star$, the equations \eqref{eq:eq_main0} satisfy $(\alpha,\tau) = (\alpha_\star,1)$. 
\end{lem}
We leave for future work coming up with sufficient conditions on $P_Z$ under which $(\Lm_\star,P_Z)\in\Cc_{\rm lin}$, which would imply that the bound of Theorem \ref{thm:lowerbound_reg} is achieved by choosing $\Lm=\Lm_\star$ and $\la=\la_\star$ in \eqref{eq:opt_reg_main}. In Figures \ref{fig:fig}(Left) and \ref{fig:fig_app}(Top Left), we numerically (by using gradient descent) evaluate the performance of the proposed loss function $\Lm_\star$, in the case of Laplacian noise, suggesting that it achieves the lower bound $\alpha_\star$ in Theorem \ref{thm:lowerbound_reg}. See also Figure \ref{fig:lopt}(Left) for an illustration of $\Lm_\star$.

\subsection{The Sub-optimality Gap of RLS in Linear Models}\label{sec:LS_linear}

We rely on Theorem \ref{thm:lowerbound_reg} to investigate the statistical gap between least-squares (i.e. $\Lm(t)=t^2$ in \eqref{eq:opt_reg_main}) and the optimal choice of $\Lm$. As a first step, the lemma below computes the high-dimensional limit of optimally regularized RLS.


\begin{lem}[Asymptotic error of optimally regularized RLS]\label{cor:LS_reg}
Fix $\delta>0$ and noise distribution $P_Z$. Let $\xh\,_{{\ell_2,\la}}$ be the solution to $\la$-regularized least-squares.
Further let $\alpha_{\ell_2,\la}$ denote the high-dimensional limit of  $\left\|\xh\,_{{\ell_2,\la}}-\x_0\right\|_2^2$. Then, $\la\mapsto \alpha_{\ell_2,\la}$ is minimized at $\la_{\rm opt} = 2\,\E[Z^2]$ and  $$\alpha_{{\ell_2,\la_{\rm opt}}}^2 := h_{\delta}\left(\E\left[Z^2\right]\right).$$
\end{lem}

We combine this result with the closed-form lower bound of Corollary \ref{cor:lowerbound} to find that
${\alpha_\star^2}/{\alpha_{\,{\ell_{_2},\la_{\rm opt}}}^2} \in[\omega_{_\delta},1]$ for  $$\omega_{\delta}:=\frac{h_{{\delta}}\left(1/\Ic(Z)\right)}{h_{{\delta}}\left(\E\left[Z^2\right]\right)}.$$ 
The fact that $\omega_\delta \leq 1$ follows directly by the increasing nature of the function $h_{\delta}$ and the Cramer-Rao bound $\E[Z^2] \ge {1}\big/\Ic(Z)$ (see Proposition \ref{propo:Fisher}(c)). Moreover, using analytic properties of the function $h_{\delta}$ it is shown in Section \ref{sec:proof_omega_LB} that
\bea\label{eq:omega_LB}
{\alpha_\star^2}\big/{\alpha_{\,{\ell_{_2},\la_{\rm opt}}}^2}\geq \omega_{_\delta} \geq \max\left\{ 1-\delta \,,\, \big(\Ic(Z)\, \E[Z^2]\big)^{-1}  \right\}.
\eea

The first term in the lower bound in \eqref{eq:omega_LB} reveals that in the highly over-parameterized regime ($\delta\ll 1$), it holds $\omega_\delta\approx 1$. Thus, optimally-regularized LS becomes optimal. 
More generally, in the overparameterized regime $0<\delta<1$, the squared-error of optimally-tuned LS is no worse than $(1-\delta)^{-1}$ times the optimal performance among all convex ERM. 

The second term in \eqref{eq:omega_LB} is more useful in the underparameterized regime $\delta\geq 1$ and captures the effect of the noise distribution via the ratio $(\Ic(Z)\,\E[Z^2])^{-1} \leq 1$ (which is closely related to the classical Fisher information distance studied e.g. in \cite{johnson2004fisher}). From this and the fact that $\Ic(Z) = 1/\E[Z^2]$ iff $Z\sim\mathcal{N}(0,\zeta^2)$ we conclude that $\omega_{\delta}$ attains its maximum value $1$ (thus, optimally-tuned LS is optimal) when $Z$ is Gaussian.
%
For completeness, we remark that \cite{wu2012optimal} has shown that when $Z\sim  \mathcal{N}(0,\zeta^2)$, then the minimum mean square error (MMSE) is also given by $h_{\delta}(\zeta^2)$. 
To further illustrate that our results are informative for general noise distributions, consider the case of Laplacian noise, i.e., $Z\sim \texttt{Laplace}(0,b^2)$. Using $\E[Z^2] = 2b^2$ and $\Ic(Z) = b^{-2}$ in \eqref{eq:omega_LB} we obtain $\omega_{\delta}\ge 1/2$, for all $b>0$ and $\delta>0$. Therefore we find that optimally-tuned RLS achieves squared-error that is at most twice as large as the optimal error, i.e. if $Z\sim  \texttt{Laplace}(0,b^2),~b>0$ then for all $\delta>0$ it holds that $\alpha_{\,{\ell_{_2},\la_{\rm opt}}}^2 \le 2\,\alpha_\star^2$. See also Figures \ref{fig:fig} and \ref{fig:fig_app} for a numerical comparison.

\section{Binary Models}\label{sec:binary}
Consider data $(y_i,\ab_i),~i\in[m]$ from a binary model:
$y_i = f(\ab_i^T\x_0)$ where $f$ is a (possibly random) link function outputting $\{\pm1\}.$

\begin{ass}[Link function]\label{ass:label}
The link function $f$ satisfies $\nu_f:=\E\left[S\,f(S)\right]\neq0,$ for $S\sim \mathcal{N}(0,1)$.\footnote{See Section \ref{sec:sfs} for further discussion.}
\end{ass}
Under Assumptions \ref{ass:HD}, \ref{ass:gaussian} and \ref{ass:label} we study the ridge-regularized ERM for binary measurements: 
\begin{equation}\label{eq:opt_bin_main}
\wh_{{\Lm,\la}}:=\arg \min_{\w\in\R^n}\;\;\frac{1}{m}\sum_{i=1}^m \Lm\left(y_i\ab_i^T \w\right)+\frac{\la}{2}\|\w\|^{2}.
\end{equation}
We also assume that $\|\x_0\|_2=1$ since the signal strength can always be absorbed in the link function, i.e., if $\|\x_0\|_2=r>0$ then the results continue to hold for a new link function $\widetilde{f}(t) := f\big(r t\big)$.
\subsection{Asymptotic Performance}\label{sec:background_bin}
In contrast to linear models where we focused on squared error, for binary models, a more relevant performance measure is normalized correlation $\corr{\wh_{{\Lm,\la}}}{\x_0}$.
Our first result determines the limit of $\corr{\wh_{{\Lm,\la}}}{\x_0}$. Specifically, we show that for a wide class of loss functions it holds that
\bea\label{eq:corr_lim}
\rho_{_{\Lm,\la}}:=\corr{\wh_{{\Lm,\la}}}{\x_0}:=\frac{|\wh_{{\Lm,\la}}^T\,\x_0|}{\|\wh_{{\Lm,\la}}\|_2\|\x_0\|_2} \rP \sqrt\frac{1}{1+\sig_{{\Lm,\la}}^2}\,,
\eea
where $\sig_{{\Lm,\la}}^2:=\alpha_{{\Lm,\la}}^2\big/\mu_{{\Lm,\la}}^2$ and $(\alpha_{{\Lm,\la}},\mu_{{\Lm,\la}})$ are found by solving the following system of three nonlinear equations in three unknowns $(\alpha,\mu,\tau)$, for $G,S\simiid\mathcal{N}(0,1)$ :
\begin{subequations}\label{eq:bin_sys}
\begin{align}
 \Exp\Big[S\,f(S)\,\envdx{\Lm}{\ourx}{\tau} \Big]&=-\la\mu,\quad\\
{\tau^2}\,{\delta}\,\Exp\Big[\left(\envdx{\Lm}{\ourx}{\tau}\right)^2\Big]&=\alpha^2,\\
{\tau\,\delta}\,\E\Big[ G\, \envdx{\Lm}{\ourx}{\tau}  \Big]&=\alpha(1-\la\tau\delta).
\end{align}
\end{subequations}
To formalize this, we define the following collection of loss and link functions:
\begin{equation}\label{eq:Cbin}
\begin{split}
\mathcal{C}_{\rm bin} := \Big\{(\Lm , \,f )\Big|\,\text{$\forall\la>0$: \eqref{eq:bin_sys} has a unique bounded solution $(\alpha_{\Lm,\la}>0,\mu_{{\Lm,\la}},\tau_{{\Lm,\la}}>0)$} \text{ and \eqref{eq:corr_lim} holds}\Big\}.
\end{split}
\end{equation}
\begin{thm}[Asymptotics for binary RERM]\label{propo:boundedness} Let Assumptions \ref{ass:HD} and \ref{ass:gaussian} hold and $\|\x_0\|_2=1$. Let $f:\R\rightarrow\{-1,+1\}$ be a link function satisfying Assumption \ref{ass:label}. Further assume a loss function $\Lm$ with the following properties: $\Lm$ is convex, twice differentiable and bounded from below such that $\Lm^\prime(0)\neq0$ and for $G\sim\mathcal{N}(0,1)$, we have $\E[\Lm(G)] < \infty$. Then, it holds that $(\Lm,f) \in \mathcal{C}_{\rm bin}.$
\end{thm}

We prove Theorem \ref{propo:boundedness} in Section \ref{sec:asy_bin}. Previous works have considered special instances of this: \cite{sur2019modern,salehi2019impact} study unregularized and regularized logistic-loss for the logistic binary model, while \cite{taheri2020sharp} studies strictly-convex ERM without regularization. Here, we follow the same approach as in \cite{salehi2019impact,taheri2020sharp}, who apply the convex Gaussian min-max theorem (CGMT) to relate the performance of RERM to an auxiliary optimization (AO) problem whose first-order optimality conditions lead to the system of equations in \eqref{eq:bin_sys}. Our technical contribution in proving Theorem \ref{propo:boundedness} is proving existence and uniqueness of solutions to \eqref{eq:bin_sys} for a broad class of convex losses.
As a final remark, the solution to \eqref{eq:bin_sys} (specifically, the parameter $\sig_{{\Lm,\la}}^2$) further determines the high-dimensional limit of the classification error for a fresh feature vector $\ab\sim\Nn(\mathbf{0},\mathbf{I}_n)$ (see  Section \ref{sec:proofoferr}) :
\bea\label{eq:testerror2}
\mathcal{E}_{{\Lm,\la}} := \Pro \left( f\left(\ab^T \x_0\right) \left(\ab^T\wh_{{\Lm,\la}}\right) < 0\right)\rP\Pro \left( \sig_{{\Lm,\la}}\,G + Sf(S) < 0 \right),~~G,S\simiid \mathcal{N}(0,1).
\eea

\subsection{Fundamental Limits and Optimal Tuning}\label{sec:limitandopt}
Thus far, we have shown in \eqref{eq:corr_lim} and \eqref{eq:testerror2} that $\sig_{{\Lm,\la}}$ predicts the high-dimensional limit of the correlation and classification-error of the RERM solution $\wh_{{\Lm,\la}}$. In fact, smaller values for $\sig_{{\Lm,\la}}$ result in better performance, i.e. higher correlation and classification accuracy (see Section \ref{sec:proofoferr}). In this section we derive  a lower bound on $\sig_{{\Lm,\la}}$ characterizing the statistical limits of RERM for binary models.
%

\begin{thm}[Lower Bound on ${\sig_{{\Lm,\la}}}$]\label{thm:lowerbound_bin}
Let Assumptions \ref{ass:HD}, \ref{ass:gaussian} and \ref{ass:label} hold. For $G,S\simiid \mathcal{N}(0,1)$ define the random variable $W_s:= s\,G + S\,f(S)$ parameterized by $s\in\R$.
Fix any $\delta>0$ and define 
\bea\label{eq:sigopt_thm}
\sig_{\star}= \sig_\star(\delta,f):=\min_{0\le x<1/\delta}\left[s > 0:  \frac{1-s^2(1-s^2\Ic(W_s))}{\delta s^2(s^2\Ic(W_s)+\Ic(W_s)-1)}-2x + x^2\delta\big(1+\frac{1}{s^2}\big) = 1\right].
\eea
For any $(\Lm,f)\in\mathcal{C}_{\rm bin}$, $\la>0$ and $\sig_ {{\Lm,\la}}^2$ the respective high-dimensional limit of the error as in \eqref{eq:corr_lim}, it holds that $\sig_ {{\Lm,\la}} \geq \sig_\star$.

\end{thm}

We prove Theorem \ref{thm:lowerbound_bin} in Section \ref{sec:proofofbin}, where we also show that the minimization in \eqref{eq:sigopt_thm} is always feasible. In view of \eqref{eq:corr_lim} and \eqref{eq:testerror2} the theorem's lower bound translates to an upper bound on correlation and test accuracy. Note that $\sigma_\star$ depends on the link function only through the Fisher information of the random variable $s\,G + S\,f(S)$. This parallels the lower bound of Theorem \ref{thm:lowerbound_reg} on linear models with the random variable $S\,f(S)$ effectively playing the role of the noise variable $Z$.

\vp
Next we present  
a useful closed-form lower bound for $\sig_\star$. For convenience define the function $H_\delta : \R_{>1}\rightarrow\R_{>0}$ parameterized by $\delta>0$ as following,
\bea
H_{\delta}(x):= 2\left(-\delta-x+\delta \,x+\sqrt{(-\delta-x+\delta \,x)^2 + 4\delta(x-1)}\right)^{-1}.
\eea

\begin{cor}[Lower bound on $\sig_\star$]\label{cor:lowerbound_binary}
Let $\sig_\star$ be as in \eqref{eq:sigopt_thm}. Fix any $\delta>0$ and assume that $f(\cdot)$ is such that the random variable $Sf(S)$ has a differentiable and strictly positive probability density on the real line. Then, 
$$\sig_\star^2\ge H_\delta \left(\;\Ic(Sf(S))\;\right).$$
\end{cor}
Corollary \ref{cor:lowerbound_binary} can be viewed as an extension of Corollary \ref{cor:lowerbound} to binary models. The proof of the corollary presented in Section \ref{sec:proofofcor_bin} further reveals that the more the distribution of $Sf(S)$ resembles a Gaussian distribution, the tighter the gap is, with equality being achieved if and only if $Sf(S)$ is Gaussian.

\vp
Our next result strengthens the lower bound of Theorem \ref{thm:lowerbound_bin} by showing existence of a loss function and regularizer parameter for which the system of equations \eqref{eq:bin_sys} has a solution leading to $\sigma_\star$.

\begin{lem}[{Optimal tuning for binary RERM}]\label{thm:opt_bin}
For given $\delta>0$ and binary link function $f$, let $(\sigma_{\star}>0,x_\star\in[0,1/\delta))$ be the optimal solution in the minimization in \eqref{eq:sigopt_thm}. Denote $\la_\star=x_\star$ and define $W_\star := \sigma_\star G + Sf(S)$. Consider the loss function
$\Lm_{\star}:\R\rightarrow\R$
\bea\label{eq:optimalloss_thm}
\Lm_{\star}(x) := -\env{\frac{\eta(\la_{\star}\delta-1)}{\delta(\eta-\Ic(W_{_{\star}}))}Q + \frac{\la_{\star}\delta-1}{\delta(\eta-\Ic(W_{_{\star}}))}\log\left(p_{_{W_{_{\star}}}}\right)}{x}{1},
\eea
where 
$
\eta:=1- \Ic(W_{{\star}})\cdot(\sig_{\star}^2-\sig_{\star}^2\la_{\star}\delta - \la_{\star}\delta)-\la_{\star}\delta$ and $Q(w) := w^2/2.$
Then for $\Lm_\star$ and $\la_\star$, the equations \eqref{eq:bin_sys} satisfy $(\alpha,\mu,\tau) = (\sig_\star,1,1)$. 
\end{lem}
Lemma \ref{thm:opt_bin} suggests that if $\Lm_\star$ satisfies the assumptions of Theorem \ref{propo:boundedness}, then $\sig_{\Lm_\star,\la_\star} = \sig_\star$. In Figures \ref{fig:fig} and \ref{fig:fig_app} and for the special cases of Signed and Logistic models, we verify numerically that performance of candidates $\Lm_\star$ and $\la_\star$ reaches the optimal errors . This suggests that for these models, Lemma \ref{thm:opt_bin} yields the optimal choices for $\Lm$ and $\la$. See also Figure \ref{fig:lopt}(Right) for an illustration of $\Lm_\star$.

\subsection{The Sub-optimality Gap of RLS in Binary Models}\label{LS_binary}
We use the optimality results of the previous section to precisely quantify the sub-optimality gap of RLS. First, the following lemma characterizes the performance of RLS.
\begin{lem}[Asymptotic error of RLS]\label{cor:LS_bin}
Let Assumptions \ref{ass:HD}, \ref{ass:gaussian} and \ref{ass:label} hold. Recall that $\nu_f=\E[Sf(S)]\neq 0$. Fix any $\delta>0$ and consider solving \eqref{eq:opt_bin_main} with the square-loss $\Lm(t)=(t-1)^2$ and $\la\geq0$. Then, the system of equations in \eqref{eq:bin_sys} has a unique solution $(\alpha_{{\ell_2,\la}},\mu_{{\ell_2,\la}},\tau_{{\ell_2,\la}})$ and 
\bea\label{eq:sigmaLSreg}
\sig_{{\ell_2,\la}}^2= \frac{\alpha_{\ell_2,\la}^2}{\mu_{\ell_2,\la}^2} = \frac{1}{2\delta \nu_f^2}\Big(1-\delta \nu_f^2 + \frac{2+ 2\delta+\la\delta+\delta\nu_f^2\left((2+\la)\delta-6\right)}{\sqrt{4+4\delta(\la-2)+\delta^2(\la+2)^2}}\Big).
\eea
Moreover, it holds that
$\sig_{{\ell_2,\la}}^2 \geq \sig_{{\ell_2,\la_{\rm opt}}}^2 := H_\delta \big((1-\nu_f^2)^{-1}\big)$
with equality attained for the optimal tuning $\la_{{\rm opt}}= 2\big(1-\nu_f^2\big)\big/\big(\delta\,\nu_f^2\big)$.
\end{lem}
In resemblance to Lemma \ref{cor:LS_reg} in which RLS performance for linear measurements only depends on the second moment $\E[Z^2]$ of the additive noise distribution, Lemma \ref{cor:LS_bin} reveals that the corresponding key parameter for binary models is $1-\nu_f^2$. Interestingly, the expression for $\sig_{{\ell_2,\la_{\rm opt}}}^2$ conveniently matches with the simple bound on $\sig_\star^2$ in Corollary \ref{cor:lowerbound_binary}. Specifically, it holds for any $\delta>0$ that
\bea\label{eq:RLS_bin}
1\;\ge\;\frac{\sig_\star^2}{\sig_{{\ell_{_2},\la_{\rm opt}}}^2} \ge\Omega_{\delta}:=\frac{H_{\delta}\left(\,\Ic(S\,f(S))\,\right)}{H_{\delta}\left((1-\nu_f^2)^{-1}\right)}.
\eea
We note that $H_\delta(\cdot)$ is strictly-decreasing in its domain for a fixed $\delta>0$. Furthermore, the Cramer-Rao bound (see Prop. \ref{propo:Fisher} (d)) requires that $\Ic(Sf(S))\geq \left({\rm Var}[Sf(S)]\right)^{-1} =\big(1-\nu_f^2\big)^{-1}$. Combining these,  confirms that $\Omega_\delta\le1$. Furthermore $\Omega_\delta=1$ and thus $\sig_\star^2 = \sig_{{\ell_{_2},\la_{\rm opt}}}^2$ iff the random variable $Sf(S)$ is Gaussian. 
However, for any binary link function satisfying Assumption \ref{ass:label}, $Sf(S)$ does not take a Gaussian distribution (see Section \ref{sec:sfs}), thus \eqref{eq:RLS_bin} suggests that square-loss cannot be optimal. Nevertheless, one can use \eqref{eq:RLS_bin} to argue that square-loss is (perhaps surprisingly) approximately optimal for certain popular models. 

For instance consider logistic link function $\widetilde{f}_{ r}$ defined as $\Pro(\widetilde{f}_{ r}(x)=1) = (1+\exp(-rx))^{-1}$, where $r:=\|\x_0\|_2$.
Using \eqref{eq:RLS_bin} and maximizing the sub-optimality gap $1/\Omega_\delta$ over $\delta>0$, we find that if $f=\widetilde{f}_{r=1}$ then for all $\delta>0$ it holds that $$\sig_{{\ell_{_2},\la_{\rm opt}}}^2 \le 1.003 \; \sig_\star^2.$$ \emph{Thus, for a logistic link function and $\|\x_0\|_2=1$ optimally-tuned RLS is approximately optimal!} This is in agreement with the key message of Corollary \ref{cor:lowerbound_binary} on the critical role played by $Sf(S)$, since for the logistic model and small values of  $r$, its density is ``close'' to a Gaussian. 
However, as signal strength increases and $\widetilde{f}_{r}$ converges to the sign function ($\widetilde{f}_r(\cdot)\rightarrow \sign(\cdot)$), there appears to be room for improvement between RLS and what Theorem \ref{thm:lowerbound_bin} suggests to be possible. This can be precisely quantified using \eqref{eq:RLS_bin}. For example, for  $r=10$ it can be shown that $ \sig_{{\ell_{_2},\la_{\rm opt}}}^2 \le 2.442 \; \sig_\star^{2}, ~\forall \delta>0$. Lemma \ref{thm:opt_bin} provides the recipe to bridge the gap in this case. Indeed, Figures \ref{fig:fig} and \ref{fig:fig_app} show that the optimal loss function $\Lm$ predicted by the lemma outperforms RLS for all values $\delta$ and its performance matches the best possible one specified by Theorem \ref{thm:lowerbound_bin}.
 
\section{Numerical Experiments}
In Figure \ref{fig:fig}(Left), we compare the lower bound of Theorem \ref{thm:lowerbound_reg} with the error of RLS (see Lemma \ref{cor:LS_reg}) for $Z\sim \texttt{Laplace}(0,1)$ and $\|\x_0\|_2=1$. To numerically validate that $\alpha_\star$ is achievable by the proposed choices of loss function and regularization parameter in Lemma \ref{thm:opt_reg}, we proceed as follows. We generate noisy linear measurements with iid Gaussian feature vectors $\ab_i\in\R^{100}$. The estimator $\xh_{{\Lm_\star,\la_\star}}$ is computed by running gradient descent (GD) on the corresponding optimization in \eqref{eq:opt_reg_main} when the proposed optimal loss and regularizer of Lemma \ref{thm:opt_reg} are used. See Figure \ref{fig:lopt}(Left) for an illustration of the optimal loss for this model. The resulting vector $\xh_{{\Lm_\star,\la_\star}}$ is used to compute $\|\xh_{{\Lm_\star,\la_\star}} - \x_0\|^2$. The average of these values over 50 independent Monte-carlo trials is shown in red squares.
The close match between the theoretical and empirical values suggest that the fundamental limits presented in this paper are accurate even in small dimensions (also see the first and second rows of Table \ref{table:ratio}).

In the next two figures, we present results for binary models. Figure \ref{fig:fig}(Middle) plots the effective error parameter $\sig$ for the Signed model and Figure \ref{fig:fig}(Right)  plots the classification error `$\mathcal{E}$' for the Logistic model with $\|\x_0\|_2 =10$.
The red squares correspond to the numerical evaluations of ERM with $\Lm= \Lm_\star$ and $\la=\la_\star$ (as in Lemma \ref{thm:opt_bin}) derived by running GD on the proposed optimal loss and regularization parameter. See Figure \ref{fig:lopt}(Right) for an illustration of the optimal loss in this case. The solution $\wh_{{\Lm_\star,\la_\star}}$ of GD is used to calculate $\sig_{{\Lm_\star,\la_\star}}$ and $\mathcal{E}_{{\Lm_\star,\la_\star}}$ in accordance with \eqref{eq:corr_lim} and \eqref{eq:testerror2}, respectively. Again, note the close match between theoretical and numerical evaluations (also see the third and fourth rows of Table \ref{table:ratio}).

Finally, for all three models studied in Figure \ref{fig:fig}, we also include the theoretical predictions for the error of the following: (i) RLS with small and large regularization (as derived in Equations \eqref{eq:alphaLSreg} and \eqref{eq:sigmaLSreg}); (ii) optimally tuned RLS (as predicted by Lemmas \ref{cor:LS_reg} and \ref{cor:LS_bin}); (iii)  optimally-tuned unregularized ERM (marked as $\alpha_{\rm ureg}, \sig_{\rm ureg}, \Ec_{\rm ureg}$). 
The curves for the latter are obtained from \cite{bean2013optimal} and \cite{taheri2020sharp} for linear and binary models, respectively. We refer the reader to Sections \ref{sec:gains_lin} and \ref{sec:gains_bin} for a precise study of the benefits of regularization in view of Theorems \ref{thm:lowerbound_reg} and \ref{thm:lowerbound_bin}, for both linear and binary models.

\begin{figure}
\centering
\begin{subfigure}{.33\textwidth}
  \centering
  \includegraphics[width=\linewidth,height= 4.8cm]{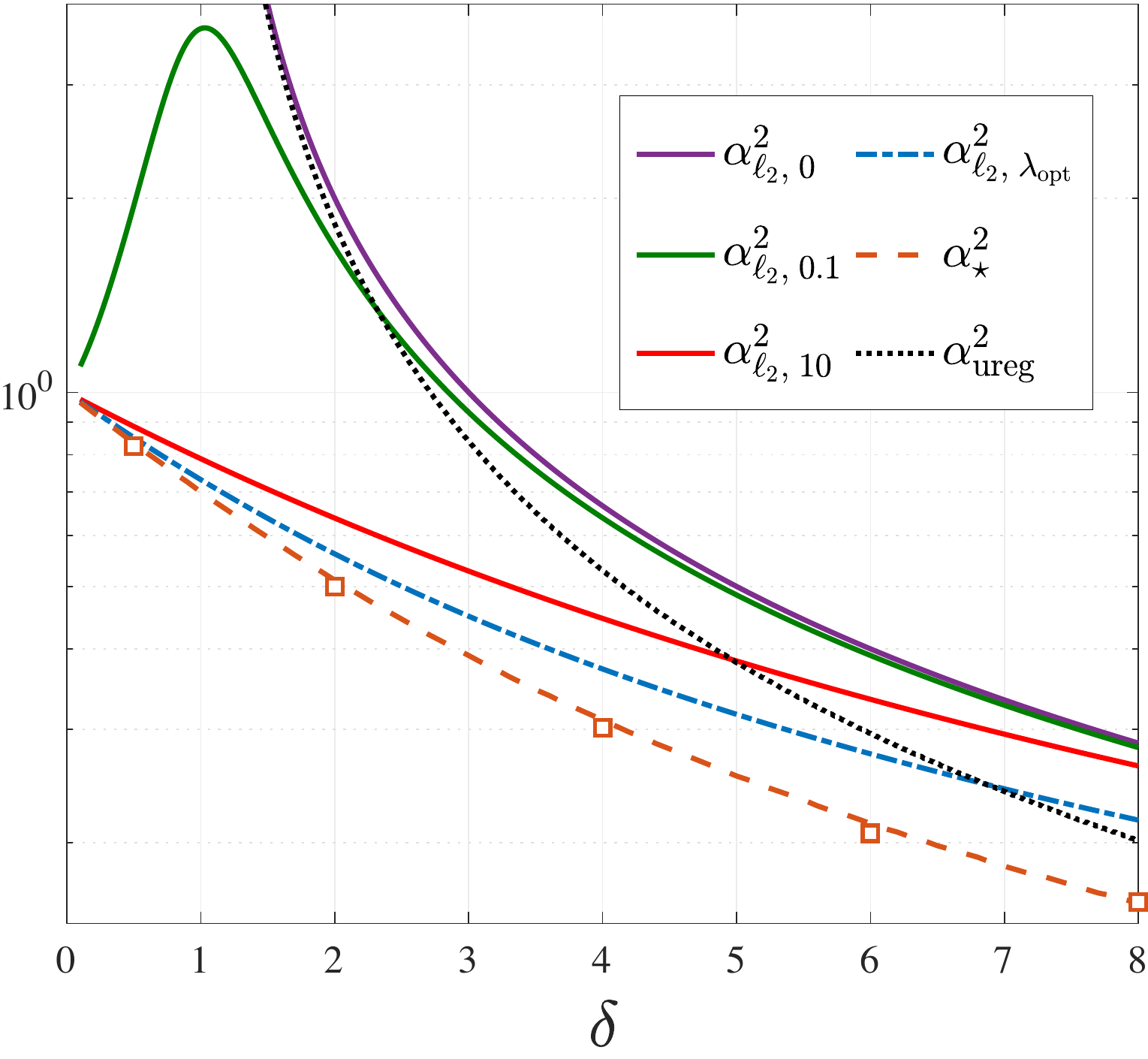}
  \label{fig:figure2}
\end{subfigure}
\begin{subfigure}{0.33\textwidth}
  \centering
  \includegraphics[width=0.99\linewidth,height= 4.8cm]{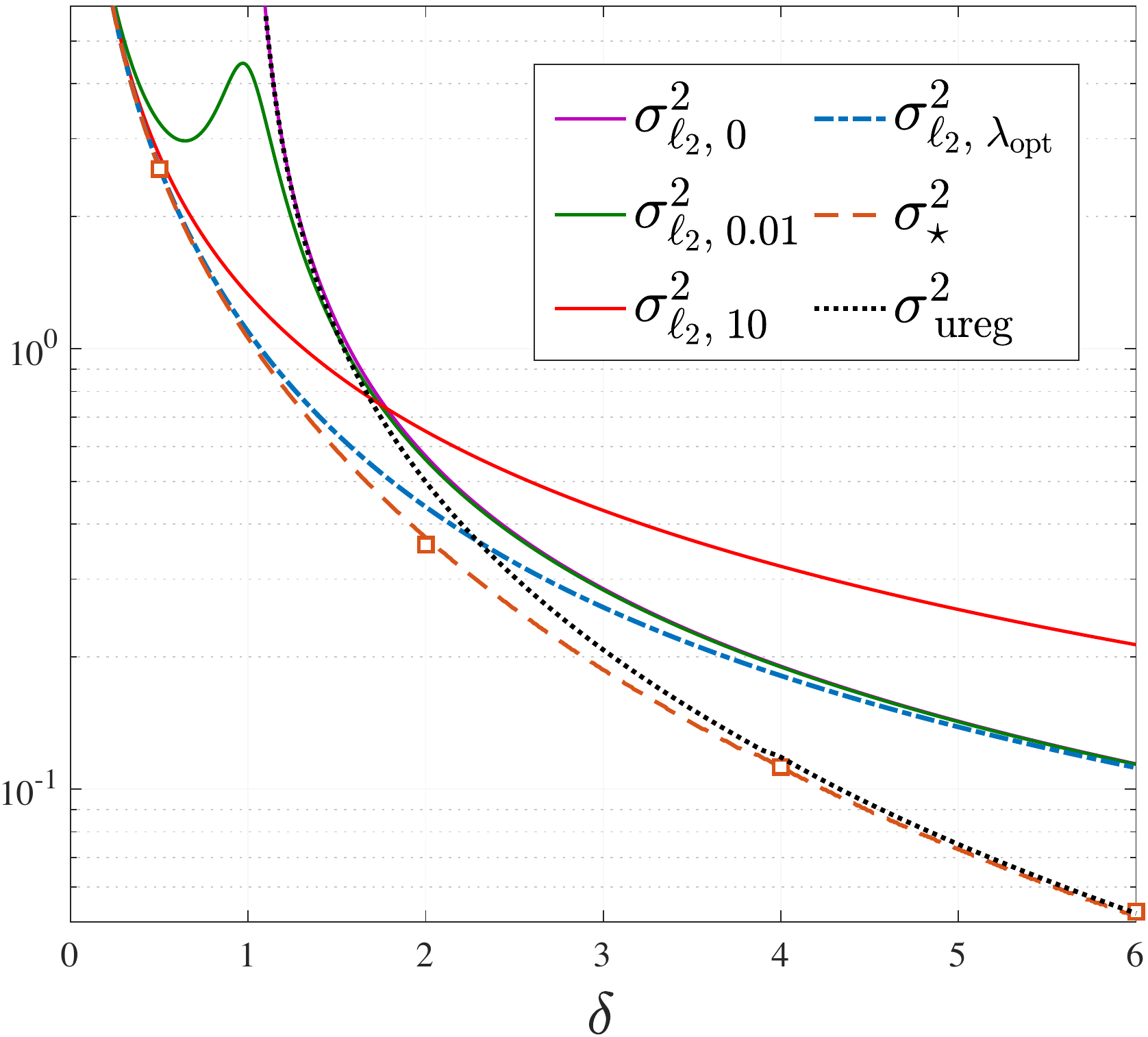}
  \label{fig:fig_sign}
\end{subfigure}%
\begin{subfigure}{0.33\textwidth}
  \centering
  \includegraphics[width=0.99\linewidth,height= 4.8cm]{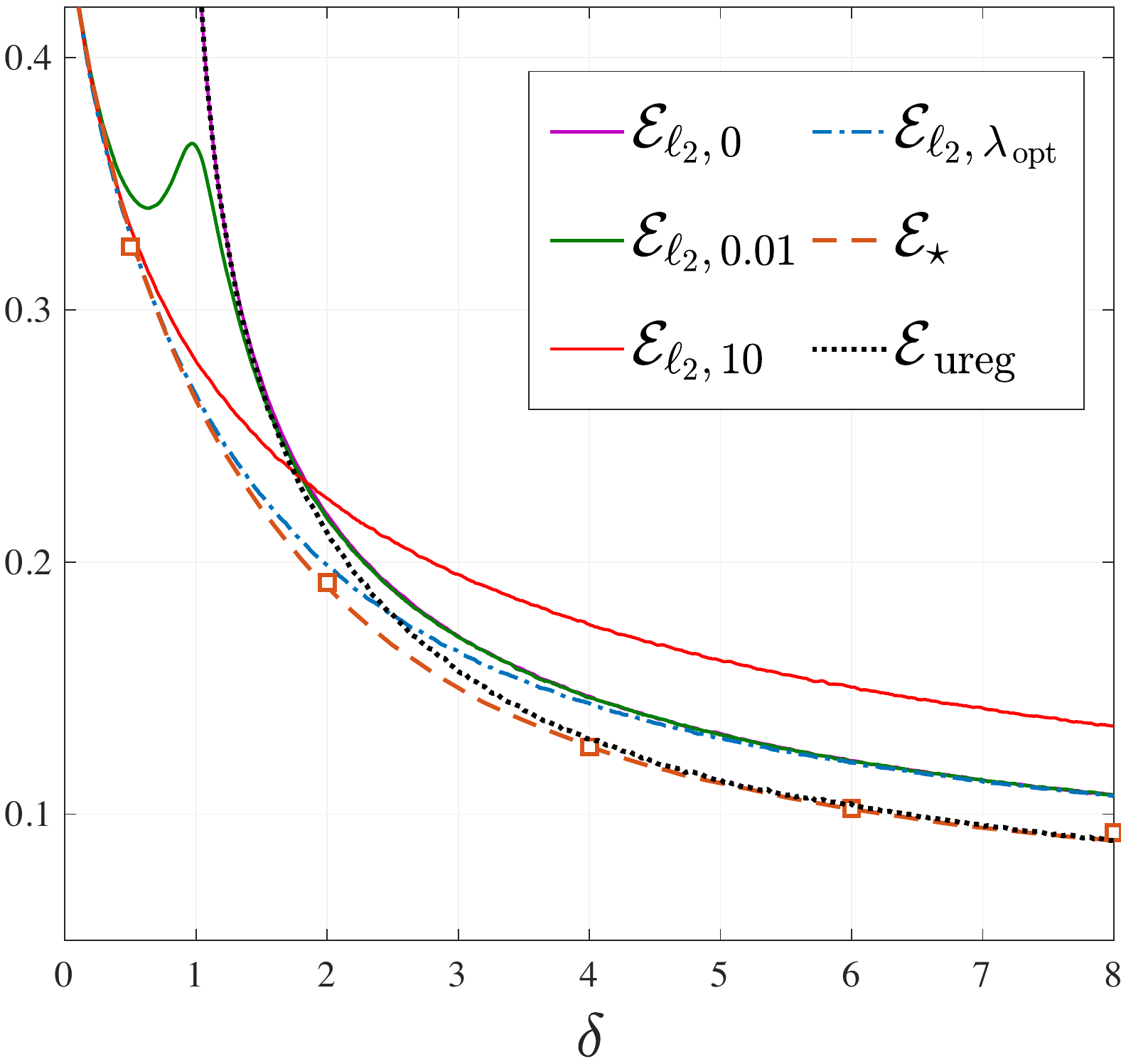}
  \label{fig:fig_log}
\end{subfigure}
\caption{The lower bounds on error derived in this paper, compared to RLS for the linear model with $Z\sim\texttt{Laplace}(0,1)$ (Left), and for the binary Signed model (Middle) and binary Logistic model with $\|\x_0\|=10$ (Right). The red squares denote the performance of the optimally tuned RERM as derived in Lemmas \ref{thm:opt_reg} and \ref{thm:opt_bin}. See Section \ref{sec:numeric} for additional numerical results.}
\label{fig:fig}
\end{figure}
 \begin{table}
\caption{Theoretical and numerical values of $\alpha_\star^2/\alpha_{\Lm,\la_{\rm opt}}^2$ (for linear models) and $\sig_\star^2/\sig_{\Lm,\la_{\rm opt}}^2$(for binary models) for different values of $\delta$ and for some special cases studied in this paper. The theoretical results for $\alpha_\star$ and $\sig_\star$ correspond to Theorems \ref{thm:lowerbound_reg} and \ref{thm:lowerbound_bin}. The empirical values of $\alpha_\star$ and $\sig_\star$ are derived by numerically solving the optimally-tuned RERM (as derived in Lemmas \ref{thm:opt_reg} and \ref{thm:opt_bin}) by GD with $n=100$. Results shown are averages over $50$ independent experiments.}
\label{table:ratio}
\vskip 0.1in
\begin{center}
\begin{small}
\begin{sc}
\begin{tabular}{l  c | c  c  c  c  c  r  r}
\toprule
& $\delta$ & 0.5 &2& 4& 6 & 8 \\
\toprule
 \multirow{2}{8em}{$Z\sim\texttt{Laplace}(0,1)$} & Theory &0.9798  & 0.9103 & 0.8332& 0.7690&0.7447\\
    &Experiment  &0.9700  &0.8902  & 0.8109& 0.7530&0.7438\\
\midrule
 \multirow{2}{8em}{$Z\sim\texttt{Laplace}(0,2)$} & Theory &0.9832   & 0.9329  &  0.8796  &  0.8371  &  0.8043\\
 &Experiment&0.9785    &0.9103  &  0.8550  &  0.8316  &  0.7864\\
\midrule
 \multirow{2}{8em}{$f = \texttt{Sign}$} & Theory  &  0.9934   & 0.8531  &  0.6199  &  0.4602  &  0.3618\\
&Experiment &0.9918  &  0.8204  &  0.6210  &  0.4710  &  0.3829\\
\midrule
 \multirow{2}{10em}{$f = \texttt{Logistic}, {\small \|\x_0\|=10}$} & Theory &0.9826  &  0.8721   & 0.7116   & 0.6211  &  0.5712\\
&Experiment & 0.9477  &  0.8987 &   0.7112  &  0.6211 &   0.6389\\
\bottomrule
\end{tabular}
\end{sc}
\end{small}
\end{center}
\vskip -0.1in
\end{table}

\vspace{-0.1in}
\section{Conclusion and Future work}
This paper derives fundamental lower bounds on the statistical accuracy of ridge-regularized ERM (RERM) for linear and binary models in high-dimensions. It then derives simple closed-form approximations that allow precisely quantifying the sub-optimality gap of RLS. In Section \ref{sec:unreg_opt} in the supplementary material, these bounds are further used to study the benefits of regularization by comparing (RERM) to un-regularized ERM.

Among several interesting directions of future work, we highlight the following. First, our lower bounds make it possible to compare RERM to the optimal Bayes risk \cite{barbier2019optimal,reeves2019replica}. Second, it is interesting to extend the analysis to GLMs for arbitrary link functions beyond linear and binary studied here. A third exciting direction is investigating the fundamental limits of RERM in the presence of correlated (Gaussian) features.


\section*{Acknowledgment}
This work was supported by NSF Grant CCF-1909320 and Academic Senate Research Grant from UCSB.

\bibliographystyle{alpha}
\bibliography{main}
\clearpage

\appendix
\section{Useful facts}

\subsection{On Moreau Envelopes}
In Proposition \ref{propo:mor}, some of the differential properties of Moreau-envelope functions, used throughout the paper are summarized (cf. \cite{rockafellar2009variational}):
\begin{propo}[Properties of Moreau-envelopes]\label{propo:mor} Let $\Lm$ be a lower semi-continuous and proper function. 
Then

\noindent{(a)} The  value $\env{\Lm}{x}{\tau}$ is finite and depends continuously on $(x,\tau)$, with $\env{\Lm}{x}{\tau} \rightarrow \Lm(x)$ as $\tau\rightarrow 0_+$ and $\env{\Lm}{x}{\tau} \rightarrow \min_{t\in\R}\Lm(t)$ as $\tau\rightarrow+\infty$, for all $x\in\R$.\\

\noindent{(b)} The first order derivatives of the Moreau-envelope of a function $\Lm$ are derived as follows: 
\bea
\envdx{\Lm}{x}{\tau}&:=\frac{\partial{\env{\Lm}{x}{\tau}}}{\partial x}= \frac{1}{\tau}{(x-\prox{\Lm}{x}{\tau})},\label{eq:mor_der1}\\
\envdla{\Lm}{x}{\tau}&:=\frac{\partial{\env{\Lm}{x}{\tau}}}{\partial \tau} = -\frac{1}{2\tau^2}{(x-\prox{\Lm}{x}{\tau})^2}\label{eq:mor_der2}. 
\eea
Also if $\Lm$ is differentiable then
\bea
\envdx{\Lm}{x}{\tau}&= \Lm^\prime(\prox{\Lm}{x}{\tau})\label{eq:envdxp},\\
\envdla{\Lm}{x}{\tau}&= -\frac{1}{2}(\Lm^\prime(\prox{\Lm}{x}{\tau})^2.
\eea

\noindent{(c)} Additionally, based on the relations above, if $\Lm$ is twice differentiable then the following is derived for its second order derivatives : 
\begin{align}
\envddx{\Lm}{x}{\tau}  &= \frac{\Lm''(\prox{\Lm}{x}{\tau})}{1+\tau \Lm''(\prox{\Lm}{x}{\tau})},\label{eq:secondderivative_mor}\\[5pt]
\mathcal{M}^{''}_{\Lm,2}\left(x ; \tau\right)&= \frac{\Big(\Lm^\prime(\prox{\Lm}{x}{\tau})\Big)^2\,\Lm^{\prime\prime}(\prox{\Lm}{x}{\tau})}{1+\tau\,\Lm^{\prime\prime}(\prox{\Lm}{x}{\tau})}.\label{eq:secondderivative_mor2}
\end{align}
\end{propo}

\par

%
The following proposition gives the recipe for inverting Moreau-envelpe of a convex function:
\begin{propo}[Inverse of the Moreau envelope]\cite[Result.\,23]{advani2016statistical}\label{propo:inverse}
For $\tau>0$ and $f$ a convex, lower semi-continuous function such that $g(\cdot) =\env{f}{\cdot}{\tau}$, the Moreau envelope can be inverted so that $f(\cdot) = -\env{-g}{\cdot}{\tau}.$
\end{propo}

\begin{lem}[e.g., \cite{taheri2020sharp}, Lemma A.1.]\label{lem:H_cvx} The function $H:\R^3\rightarrow\R$ defined as follows
\bea\label{eq:H_def}
H(x,p,\tau) = \frac{1}{2\tau}(x-p)^2,
\eea
is jointly convex in its arguments.
\end{lem}

\subsection{On Fisher Information}
In Proposition \ref{propo:Fisher} we collect some useful properties of the Fisher Information for location. For the proofs and more details, we refer the interested reader to \cite{Stam}. 
\begin{propo}[Properties of Fisher Information, \cite{Stam}]\label{propo:Fisher}
Let $X$ be a zero-men random variable with probability density $p_X$ satisfying the following conditions: (i) $p_X(x)>0, -\infty<x<\infty$; (ii) $p_X^\prime(x)$ exists; and (iii) The following integral exists:
$$
\Ic(X) = \int_{-\infty}^{\infty} {\frac{(p_X^\prime(x))^2}{p_X(x)}}\,\mathrm{d}x.
$$
The Fisher information for location $\Ic(X)$ defined above satisfies the following properties.
\begin{enumerate}[(a),leftmargin=\parindent,align=left]
\item $\Ic(X) := \E\left[(\ksi_X(X))^2\right] = \E\Big[\left(\frac{p_X^\prime(X)}{p_X(X)}\right)^2\Big].$
\item For any $c\in\R$, $\Ic(X+c) = \Ic(X)$.
\item For any $c\in\R$, $\Ic(c\,X) = \Ic(X)/c^2$.
\item (Cramer-Rao bound) $\Ic(X) \geq \frac{1}{\E[X^2]}$, with equality if and only if $X$ is Gaussian.
\item For two independent random variables $X_1, X_2$ satisfying the three conditions above and any $\theta \in [0,1]$, it holds that $\Ic(X_1+ X_2) \leq \theta^2 \Ic(X_1) + (1-\theta)^2\Ic(X_2)$. 
\item (Stam's inequality) For two independent random variables $X_1, X_2$ satisfying the three conditions above, it holds that 
\bea\label{eq:Fisher_Stam}
\Ic(X_1+X_2)\le\frac{\Ic(X_1)\cdot\Ic(X_2)}{\Ic(X_1)+\Ic(X_2)}.
\eea
Moreover equality holds if and only if $X_1$ and $X_2$ are independent Gaussian random variables. 
\end{enumerate}
\end{propo}

\begin{lem}
\label{lem:Fisher_lim}
Let $G\sim\Nn(0,1)$ and $Z$ be a random variable satisfying the assumptions of Proposition \ref{propo:Fisher}. For any $a\in\R$, use the shorhand $V_a:=a\, G + Z$. The following are true:
\begin{enumerate}[(a),leftmargin=\parindent,align=left]
\item $\lim_{a\rightarrow0}a^2\Ic(V_a) = 0.$
\item $\lim_{a\rightarrow+\infty} a^2\Ic(V_a) = 1.$
\end{enumerate}
\end{lem}
\begin{proof}To show part $(a)$, we use Proposition \ref{propo:Fisher}(e) with $\theta = 0$ to derive that
\bea\label{eq:alphatozero}
\lim_{a\rightarrow 0} a^2\,\Ic(V_a) \le \lim_{a\rightarrow 0} a^2\,\Ic(Z) = 0,
\eea
where the second step follows by the fact that $\Ic(Z)$ is finite for any $Z$ satisfying the assumption of the lemma. In order to prove part $(b)$, we apply Proposition \ref{propo:Fisher}(c) to deduce that :
\bea\label{eq:alphatoinfty}
\lim_{a\rightarrow+\infty}a^2\,\Ic(V_a) =\lim_{a\rightarrow+\infty}a^2\,\Ic(a\,G + Z) = \lim_{a\rightarrow+\infty}\Ic(G+\frac{1}{a}Z) = 1,
\eea
\end{proof}

\subsection{On Min-max Duality}
\begin{thm}[Sion's min-max theorem \cite{sion1958}]\label{lem:minmaxsion}
Let $X$ be a compact convex subset of a linear topological space and $Y$ a convex subset of a linear topological space. If $f$ is a real-valued function on $X \times Y$ with
$f(x, \cdot)$ upper semicontinuous and quasi-concave on $Y, \forall x \in X,$ and $f(\cdot, y)$ lower semicontinuous and quasi-convex on $X, \forall y \in Y$
then,
\[
\min _{x \in X}\, \sup _{y \in Y} \,f(x, y)=\sup _{y \in Y}\, \min _{x \in X} \,f(x, y).
\]
\end{thm}

\section{Asymptotics for Binary RERM: Proof of Theorem \ref{propo:boundedness}}\label{sec:asy_bin}

In this section, we prove that under the assumptions of Theorem \ref{propo:boundedness}, the system of equations in  \eqref{eq:bin_sys} has a unique and bounded solution.


\subsection{Asymptotic Error of RERM via an Auxiliary Min-Max Optimization}

As mentioned in Section \ref{sec:binary}, the proof of Theorem \ref{propo:boundedness} has essentially two parts. The first part of the proof uses the CGMT \cite{COLT} and the machinery developed in \cite{Master,NIPS,salehi2019impact,taheri2019sharp} to relate the properties of the RERM solution to an Auxiliary Optimization (AO). The detailed steps follow mutatis-mutandis analogous derivations in recent works \cite{Master,NIPS,salehi2019impact,taheri2019sharp,svm_abla} and are omitted here for brevity. Instead, we summarize the finding of this analysis in the following proposition.
\begin{propo}\label{propo:minmax}
 Consider the optimization problem in \eqref{eq:opt_bin_main}. If the min-max optimization in \eqref{eq:minmax_bin} has a unique and bounded solution $(\alpha^\star>0,\mu^\star,\upsilon^\star>0,\gamma^\star>0)$, then the values of $\alpha_{\Lm,\la}$ and $\mu_{\Lm,\la}$ corresponding to $\Lm$ and $\la$ defined in \eqref{eq:mu}-\eqref{eq:error_bin} are derived by setting $\alpha_{\Lm,\la} = \alpha_\star$ and $\mu_{\Lm,\la} = \mu_\star$, where
\begin{align}\label{eq:minmax_bin}
\begin{split}
(\alpha^\star,\mu^\star,\upsilon^\star,\gamma^\star) = \arg\min_{\substack{(\alpha,\mu,\upsilon) \in \\[2pt]\R_{\geq0}\times\R\times \R_{>0}}} \max _{\gamma \in \R_{>0}}\bigg[ \Theta(&\alpha,\mu,\upsilon,\gamma ):= \frac{\gamma\upsilon}{2}-\frac{\alpha \gamma}{\sqrt{\delta}}+\frac{\la \mu^{2}}{2}+\frac{\la \alpha^{2}}{2} + \\ &\mathbb{E}\left[\mathcal{M}_{\Lm}\left(\alpha G+\mu Sf(S) ; \frac{\upsilon}{\gamma}\right)\right]\bigg],
\end{split}
\end{align}
and $G,S\simiid\mathcal{N}(0,1)$. 
\end{propo}


The system of equations in \eqref{eq:bin_sys} is derived by the first-order optimality conditions of the function $\Theta$ based on its arguments $(\alpha,\mu,\upsilon,\gamma)$, i.e., by imposing $\nabla\Theta = \mathbf{0}$. In fact, similar to \cite{taheri2020sharp}, it only takes a few algebraic steps to simplify the four equations in $\nabla\Theta = \mathbf{0}$ to the three equations in \eqref{eq:bin_sys}.

For the rest of this section, we focus on the second part of the proof of Theorem \ref{propo:boundedness} regarding existence/uniqueness of solutions to  \eqref{eq:bin_sys}, which has not been previously studied in our setting. 

\subsection{Properties of $\Theta$ : Strict Convexity-Strict Concavity and Boundedness of Saddle Points}

We will show in Lemma \ref{lem:unique_1} that for proving uniqueness and boundedness of the solutions to \eqref{eq:bin_sys}, it suffices to prove uniqueness and boundedness of the saddle point  $(\alpha^\star,\mu^\star,\upsilon^\star,\gamma^\star)$ of $\Theta$. In fact, a sufficient condition for uniqueness of solutions in \eqref{eq:minmax_bin} is that $\Theta$ is (jointly) strictly convex in $(\alpha,\mu,\upsilon)$ and strictly-concave in $\gamma$ (e.g., see \cite[Lemma B.2.]{taheri2020sharp}). 
Lemma \ref{lem:theta_convexity}, which is key to the proof of Theorem \ref{propo:boundedness}, derives sufficient conditions on $\Lm$ guaranteeing strict convexity-strict concavity of $\Theta$ as well as conditions on $\Lm$ ensuring boundedness of $(\alpha^\star,\mu^\star,\upsilon^\star,\gamma^\star).$

\begin{lem}[Properties of $\Theta$]\label{lem:theta_convexity}
Let $\Lm(\cdot)$ be a lower semi-continuous (lsc), proper and convex function and $\la > 0$. Then the following statements hold for the function $\Theta:\R_{\ge0}\times\R\times\R_{>0}\times\R_{>0}\rightarrow \R$ in \eqref{eq:minmax_bin},
\begin{enumerate}[(a),leftmargin=\parindent,align=left]
\item  If $\Lm$ is bounded from below, then for all solutions $(\alpha^\star,\mu^\star,\upsilon^\star,\gamma^\star)$ there exists a constant $C>0$ such that $\alpha^\star\in [0,C], \mu^\star \in [-C,C]$ and $\upsilon^\star \in [0,C]$.

\item  If $\Lm$ is bounded from below and $\E[\Lm(G)]<\infty$ for $G\sim\mathcal{N}(0,1)$, then there exists a constant $C>0$ such that $\gamma^\star \in [0,C].$ 

\item In addition to the assumptions of parts (a) and (b) assume that $\Lm^\prime(0)\neq0$, then $\gamma^\star>0,\alpha^\star>0$ and $\upsilon^\star>0$.

\item If $\Lm$ is twice differentiable and non-linear, then $\Theta$ is jointly strictly-convex in $(\alpha, \mu,\upsilon)$.

%
\item  If $\Lm$ satisfies the assumptions of part (c) then $\Theta$ is strictly-concave in $\gamma$.

\end{enumerate}
\end{lem}

\subsubsection{Proof of Lemma \ref{lem:theta_convexity}}


\noindent\textbf{Statement (a).}~~Let $\widetilde{\Theta}(\alpha,\mu,\upsilon)  := \sup_{\gamma \in \R_{>0}} {\Theta}(\alpha,\mu,\upsilon,\gamma)$.  For all feasible $(\alpha,\mu,\upsilon)$ it holds 
\begin{align}
\widetilde{\Theta}\left(\alpha,\mu,\upsilon\right)\;\; &\ge\; \Theta\left(\alpha,\mu,\upsilon,1\right) \nn\\&=\;   \frac{\upsilon}{2} -\frac{\alpha}{\sqrt{\delta}} + \frac{\la(\alpha^2+\mu^2)}{2} + \E\Big[\env{\Lm}{\alpha G +\mu Sf(S)}{\upsilon}\Big].\label{eq:boundedness1}
\end{align}
Recall that $\Lm$ is bounded from below, i.e., for all $\Lm (x) \ge B, \forall x\in\R$ for some real $B$. By definition of Moreau-envelope function the same bound holds for $\mathcal{M}_{\Lm}$, i.e. for all $x\in \R$ and $y\in \R_{>0}$, we have that $\env{\Lm}{x}{y} \ge B$. Using this, we proceed from \eqref{eq:boundedness1} to derive that:
\bea\label{eq:boundedness2}
\widetilde{\Theta}\left(\alpha,\mu,\upsilon\right)\;\; \ge 
\;\; B + \frac{\upsilon}{2} -\frac{\alpha}{\sqrt{\delta}} + \frac{\la(\alpha^2+\mu^2)}{2} .
\eea
Based on \eqref{eq:boundedness2} that holds for all feasible $(\alpha,\mu,\upsilon)$ and using the fact that $\la>0$ it can be readily shown that
\begin{align*}
&\lim_{\alpha\rightarrow +\infty}\;\min_{\substack{\left(\mu,\upsilon\right)  \in\R\times \R_{>0}}} \widetilde{\Theta}\left(\alpha,\mu,\upsilon\right) = +\infty,   \quad \quad \lim_{\upsilon\rightarrow +\infty}\;\min_{\substack{\left(\alpha,\mu\right) \in \R_{\ge0}\times \R}} \widetilde{\Theta}\left(\alpha,\mu,\upsilon\right) = +\infty,
\\ &\lim_{\mu\rightarrow \pm\infty}\;\min_{\substack{\left(\alpha,\upsilon\right) \in \R_{\ge0}\times \R_{>0}}} \widetilde{\Theta}\left(\alpha,\mu,\upsilon\right) = +\infty.
\end{align*}
Thus, the function $\widetilde{\Theta}\left(\alpha,\mu,\upsilon\right)$ is level-bounded in $\R_{\geq0}\times\R\times\R_{>0}$. This implies the boundedness of solutions $(\alpha^\star,\mu^\star,\upsilon^\star)$ to \eqref{eq:minmax_bin} \cite[Thm.~1.9]{rockafellar2009variational}, as desired.


\noindent\textbf{Statement (b).}~~Under the assumptions of the lemma, we know from part $(a)$ that the set of solutions to $(\alpha^\star,\mu^\star,\upsilon^\star)$ in \eqref{eq:minmax_bin} is bounded. Thus we can apply the Min-Max Theorem \ref{lem:minmaxsion} and flip the order of minimum and maximum to write:
\bea\label{eq:boundedness3}
\min_{\substack{\left(\alpha,\mu,\upsilon\right)  \\ \in \, [0,C]\times[-C,C]\times (0,C]}} \max_{\gamma \,\in\, \R_{\geq0}}\;\; \Theta (\alpha,\mu,\upsilon,\gamma) = \max_{\gamma\,\in\,\R_{\geq0}}\Big[ \widehat{\Theta}(\gamma) := \min_{\substack{\left(\alpha,\mu,\upsilon\right)  \\ \in \, [0,C]\times[-C,C]\times (0,C]}} \Theta \left(\alpha,\mu,\upsilon,\gamma\right)\Big].
\eea
Without loss of generality, we assume $C$ large enough such that $C > \max \{1,1/\sqrt{\delta}\}$. Then, by choosing $\alpha=1,\mu=0$ and $\upsilon = 1/\sqrt{\delta}$, we find that for all $\gamma>0$:
\bea\label{eq:boundedness4}
\widehat{\Theta}(\gamma) \;\le \;\Theta\left(1,0,1/\sqrt{\delta},\gamma\right) = -\frac{\gamma}{2\sqrt{\delta}} + \frac{\la}{2} + \E\left[\env{\Lm}{G}{\frac{1}{\gamma\,\sqrt{\delta}}}\right].
\eea
Note that for any $y\in\R$:
$
\env{\Lm}{y}{\frac{1}{\gamma\,\sqrt{\delta}}} = \min_{x\in\R} \frac{\gamma\sqrt{\delta}}{2}(x-y)^2 + \Lm(x) \le \Lm(y).
$
Thus we derive from \eqref{eq:boundedness4}:
\bea\label{eq:boundedness5}
\widehat{\Theta}(\gamma) \le -\frac{\gamma}{2\sqrt{\delta}} + \frac{\la}{2} + \E\left[\,\Lm(G)\,\right].
\eea
But $\E\left[\Lm(G)\right]$ is assumed to be bounded, thus it can be concluded from \eqref{eq:boundedness5} that the function $\widehat{\Theta}(\gamma)$ is level-bounded, i.e.,
\bea
\lim_{\gamma\rightarrow+\infty} \widehat{\Theta}(\gamma) = -\infty.
\eea
This implies boundedness of the set of maximizers $\gamma^\star$, which completes the proof.

\noindent\textbf{Statement (c).}~~First, we show that $\gamma^\star>0$. On the contrary, assume that $\gamma^\star=0$. Then based on \eqref{eq:minmax_bin} and Proposition \ref{propo:mor}(a),
\bea
(\alpha^\star,\mu^\star,\upsilon^\star) = \arg\min _{\substack{\left(\alpha,\mu,\upsilon\right)  \\ \in \, [0,C]\times[-C,C]\times (0,C]}}  \left[\frac{\la \alpha^2}{2}+\frac{\la\mu^2}{2} + \min_{t\in\R} \Lm(t)\right]\nn,
\eea
implying that $\alpha^\star=\mu^\star=0$ and $\Theta(\alpha^\star,\mu^\star,\upsilon^\star,\gamma^\star) = \min_{t\in\R} \Lm(t)$. On the other hand, in this case we find that for any $\widetilde{\gamma}\in (0,C]$,
$$\Theta(\alpha^\star,\mu^\star,\upsilon^\star,\widetilde{\gamma}) = \widetilde{\gamma}\upsilon^\star + \env{\Lm}{0}{\frac{\upsilon^\star}{\widetilde{\gamma}}} \, > \min_{t\in\R} \Lm(t).$$
To deduce the inequality, we used the fact that $\env{\Lm}{0}{\tau} = \min_{t\in\R} t^2/(2\tau) + \Lm(t) > \min_{t\in\R} \Lm(t)$ for all $\tau\ge0$, provided that $\Lm(t)$ does not attain its minimum at $t=0$. Thus, since by assumption $\Lm^\prime(0)\neq0$, we deduce that $\Theta(\alpha^\star,\mu^\star,\upsilon^\star,\widetilde{\gamma})>\Theta(\alpha^\star,\mu^\star,\upsilon^\star,\gamma^\star)$, which is in contradiction to the optimality of $\gamma^\star$. This shows that $\gamma^\star>0$ for any loss function satisfying the assumptions of the lemma. Next, we prove that $\alpha^\star>0$. if $\alpha^\star=0$, then based on the optimality of $\alpha^\star$ it holds that 
$$
\frac{\partial\Theta}{\partial{\alpha}}\Big|_{{\left(\alpha^\star,\mu^\star,\upsilon^\star,\gamma^\star\right)}}\, \ge 0,
$$
thus based on \eqref{eq:minmax_bin},
\bea\label{eq:nabla_al}
\mathbb{E}\left[G \cdot\envdx{\ell}{\mu^\star\,Sf(S)}{\frac{\upsilon^\star}{\gamma^\star}} \right]-\frac{\gamma^\star}{\sqrt{\delta}} \ge 0.
\eea
Since by assumption $G$ and $S\,f(S)$ are independent and $\E[G]=0$, we deduce from \eqref{eq:nabla_al} that $\gamma^\star=0$, which is in contradiction to the previously proved fact that $\gamma^\star>0.$ This shows that $\alpha^\star>0$, as desired. Finally, we note that if $\upsilon^\star=0$, then based on \eqref{eq:minmax_bin} and in light of Proposition \ref{propo:mor}(a), we find that,
\bea
(\alpha^\star,\mu^\star,\gamma^\star) = \arg\min _{\substack{\left(\alpha,\mu\right)  \\ \in \, [0,C]\times[-C,C]}} \max_{\gamma \,\in\, (0,C]} \left[-\frac{\alpha\gamma}{\sqrt{\delta}}+\frac{\la \alpha^2}{2}+\frac{\la\mu^2}{2} + \E\Big[\Lm\left(\alpha G + \mu Sf(S)\right)\Big]\right]\nn,
\eea
which based on the decreasing nature of RHS in terms of $\gamma$, implies that either $\gamma^\star=0$ or $\alpha^\star=0$. However, we proved that both $\gamma^\star$ and $\alpha^\star$ are positive. This proves the desired result $\upsilon^\star\neq0$ and completes the proof of this part. 

\noindent\textbf{Statement (d).}~~Let $\w_1 := (\alpha_1,\mu_1,\tau_1)$ and $\w_2 :=  (\alpha_2,\mu_2,\tau_2)$ be two distinct points in the space $\R_{\ge 0} \times \R \times \R_{>0}$. We consider two cases : \\

\noindent\underline{Case \rom{1} : $(\alpha_1, \mu_1)=(\alpha_2,\mu_2)$ } \\
In this case, it suffices to show that for fixed $\alpha>0$ and $\mu$ and under the assumptions of the lemma, the function $\mathbb{E}\left[\env{\Lm}{\alpha G+\mu Sf(S)}{ \tau}\right]$ is strictly-convex in $\tau$.
Denote by $p(\alpha,\mu,\tau):=\prox{\Lm}{\ourx}{\tau}$. First, we derive second derivate of the Moreau-envelope function with respect to $\tau$ by applying \eqref{eq:secondderivative_mor2}, and further use convexity of $\Lm$ to derive that :
\begin{align}
\frac{\partial^2}{\partial \tau^2}\mathbb{E}\Big[\mathcal{M}_{\Lm}&\left(\alpha G+\mu Sf(S) ; \tau\right)\Big] \nn\\
&= \E \left[\frac{\Big(\Lm^\prime\left(p\left(\alpha,\mu,\tau\right)\right)\Big)^2\,\Lm^{\prime\prime}\left(p\left(\alpha,\mu,\tau\right)\right)}{1+\tau\,\Lm^{\prime\prime}\left(p\left(\alpha,\mu,\tau\right)\right)} \right]\ge0.\label{eq:second_der_tau}
\end{align}
Next we show that the inequality above is strict if $\Lm(\cdot)$ is a non-linear function. First we note that combining \eqref{eq:mor_der1} and \eqref{eq:envdxp} yields that for all $x\in\R$:
\bea
\Lm^\prime(\prox{\Lm}{x}{\tau})&=  \frac{1}{\tau}{(x-\prox{\Lm}{x}{\tau})},\nn\\
\Lm^{\prime\prime}(\prox{\Lm}{x}{\tau})&=\frac{1-\proxp{\Lm}{x}{\tau}}{\tau\cdot\proxp{\Lm}{x}{\tau}}.\nn
\eea
Using these relations and denoting by $p^\prime(\alpha,\mu,\tau):= \proxp{\Lm}{\ourx}{\tau}$, we can rewrite \eqref{eq:second_der_tau} as following :
\begin{align}
\frac{\partial^2}{\partial \tau^2}\mathbb{E}\Big[&\mathcal{M}_{\Lm}\left(\alpha G+\mu Sf(S) ; \tau\right)\Big] \nn\\
&=\frac{1}{\tau^3}\E\left[\frac{\Big(\ourx-p(\alpha,\mu,\tau)\Big)^2\Big(1-p^\prime(\alpha,\mu,\tau)\Big)}{p^\prime(\alpha,\mu,\tau)\Big(1+\tau\,\Lm^{\prime\prime}(p(\alpha,\mu,\tau))\Big)}\right]\label{eq:second_der_tau2}.
\end{align}
It is straightforward to see that if $\alpha>0$, then $\alpha G+\mu Sf(S)$ has positive density in the real line. Thus from \eqref{eq:second_der_tau2} we find that :
\bea\label{eq:iff}
\frac{\partial^2}{\partial \tau^2}\mathbb{E}\Big[\mathcal{M}_{\Lm}\left(\alpha G+\mu Sf(S) ; \tau\right)\Big] = 0\;\Longleftrightarrow\;\exists c\in \R \;\;\text{s.t.}\;\forall x \in \R:\quad\prox{\Lm}{x}{\tau} = x + c.
\eea
Recalling \eqref{eq:mor_der1}, we see that the condition in \eqref{eq:iff} is satisfied if and only if :
\begin{align}\label{eq:mor_lin1}
\exists c_{_1},c_{_2} \in \R : \text{s.t.} \;\; \forall x\in \R : \env{\Lm}{x}{\tau} = c_{_1}x+c_{_2}.
\end{align}
Using inverse properties of Moreau-envelope in Proposition \ref{propo:inverse}, we derive that the loss function $\Lm(\cdot)$ satisfying \eqref{eq:mor_lin1} takes the following shape,
\begin{align*}
\forall x\in\R:\quad\Lm(x) = -\env{-c_{_1}I-c_{_2}}{x}{\tau}= c_{_1}x + \frac{\tau c_{_1}^2}{2}+c_{_2}.
\end{align*}
where $I(\cdot)$ is the identity function i.e. $I(t) = t,$  $\forall t\in \R$. Therefore if $\Lm$ is non-linear function as required by the assumption of the lemma, $\mathbb{E}\left[\mathcal{M}_{\Lm}\left(\alpha G+\mu Sf(S) ; \tau \right)\right]$ has a positive second derivative with respect to $\tau$ and consequently $\Theta$ is strictly-convex in $\upsilon$.\\


\noindent\underline{Case \rom{2} : $(\alpha_1,\mu_1)\neq(\alpha_2,\mu_2)$} \\
In this case we use definition of strict-convexity to prove the claim. First, for compactness we define :
\begin{align*}
p_i :&= \prox{\Lm}{\alpha_i G+ \mu_i Sf(S)}{\tau_i} = \arg\min_{w}  \frac{1}{2\tau_i}\left(\alpha_i G+ \mu_i Sf(S) - w\right)^2 + \Lm(w),\\
\Omega(\w_i) &= \Omega(\alpha_i,\mu_i,\tau_i) :=\frac{\la \mu_i^{2}}{2}+\frac{\la \alpha_i^{2}}{2} +\mathbb{E}\Big[\env{\Lm}{\alpha_i G+\mu_i Sf(S)} {\tau_i}\Big]
\end{align*}
for $i=1,2$. Based on the way we defined the functions $\Theta$ and $\Omega$, one can see that in order to show strict-convexity of $\Theta$ in $(\alpha,\mu,\upsilon)$ it suffices to prove strict-convexity of $\Omega$ in $(\alpha,\mu,\tau)$. Let $\theta\in(0,1)$, and denote $\tau_\theta:=\theta\tau_1+\thetao\tau_2, \al_\theta:=\theta\alpha_1+\thetao\alpha_2$ and $\mu_\theta:=\theta\mu_1+\thetao\mu_2$. With this notation, 
\bea
&\Omega(\theta \w_1+\thetao \w_2)  \leq \\
&\frac{\la\mu_{\theta}^2}{2} + \frac{\la\alpha_{\theta}^2}{2} +\E\left[\,  \frac{1}{2\tau_\theta}\Big(\al_\theta G+ \mu_\theta Sf(S) - (\theta p_1+ \thetao p_2) \Big)^2 + \Lm\Big(\theta p_1 + \thetao p_2\Big)   \,\right] \nn \\[4pt]
&=\frac{\la\mu_{\theta}^2}{2} + \frac{\la\alpha_{\theta}^2}{2} +\E\Big[\, H\Big( \al_\theta G+ \mu_\theta Sf(S), \theta p_1 + \thetao p_2, \tau_\theta \Big) + \Lm\Big(\theta p_1 + \thetao p_2\Big)   \,\Big]\nn \\[4pt]
&\leq
\frac{\la\mu_{\theta}^2}{2} + \frac{\la\alpha_{\theta}^2}{2} +\nn\\
&\E\left[\, \theta H\Big(\al_1 G+ \mu_1 Sf(S), p_1, \tau_1\Big) + \thetao H\Big(\al_2 G+ \mu_2 Sf(S),  p_2, \tau_2\Big)  + \Lm\Big(\theta p_1+ \thetao p_2\Big)   \,\right].\label{eq:EME_step1}
\eea
The first inequality above follows by the definition of the Moreau envelope. The equality in the second line uses the definition of the function $H:\R^3\rightarrow\R$ in \eqref{eq:H_def}. Finally, the last inequality follows from convexity of $H$ as proved in Lemma \ref{lem:H_cvx}.\\
Continuing from \eqref{eq:EME_step1}, we use convexity of $\Lm$ to find that 
\bea
&\Omega(\theta \w_1+\thetao \w_2) \leq \frac{\la\mu_{\theta}^2}{2} + \frac{\la\alpha_{\theta}^2}{2} + \nn\\
& \E\Big[\, \theta H(\al_1 G+ \mu_1 Sf(S),  p_1,\tau_1) + \thetao H(\al_2 G+ \mu_2 Sf(S),p_2,\tau_2)  + \theta\,\Lm(p_1) + \thetao \Lm(p_2) \Big] \label{eq:EME_make_strict}
\eea
Additionally since $\la>0$ and $(\alpha_1,\mu_1)\neq(\alpha_2,\mu_2)$, we find that :
$$
 \frac{\la\mu_{\theta}^2}{2} + \frac{\la\alpha_{\theta}^2}{2} < \frac{\la (\theta \mu_1^2 + \thetao \mu_2^2)}{2} + \frac{\la (\theta \alpha_1^2 + \thetao \alpha_2^2)}{2}.
$$
Thus proceeding from \eqref{eq:EME_make_strict} we conclude strict-convexity of the function $\Omega$ :
\begin{align*}
\Omega(\theta &\w_1+\thetao \w_2) < \frac{\la (\theta \mu_1^2 + \thetao \mu_2^2)}{2} + \frac{\la (\theta \alpha_1^2 + \thetao \alpha_2^2)}{2} + \\
&\E\Big[\, \theta H(\al_1 G+ \mu_1 Sf(S),  p_1,\tau_1) + \thetao H(\al_2 G+ \mu_2 Sf(S),p_2,\tau_2)  + \theta\,\Lm(p_1) + \thetao \Lm(p_2) \Big] \\
&= \theta \Omega(\w_1) + \thetao \Omega(\w_2).
\end{align*}
This completes the proof of part (d).

\noindent\textbf{Statement (e).}~~Based on the proof of part $(c)$ and under the assumptions of the lemma we have $\alpha^\star \neq 0$. Thus we see that the random variable $\alpha G+\mu Sf(S)$ has a positive probability density everywhere in the desired domain of the optimization problem in \eqref{eq:minmax_bin}. Next, we use the result in \cite[Proposition A.6]{taheri2020sharp}, which states that if the random variable $X$ has a positive density everywhere and $\Lm$ is continuously differentiable with $\Lm^\prime (0) \neq 0$ then 
$$
\E\Big[\env{\Lm}{X}{1/\gamma}\Big]
$$ 
is strictly concave in $\gamma$. Based on this, $\Theta$ is strictly-concave in $\gamma$. This completes the proof of the lemma. 
%
\subsection{From \eqref{eq:minmax_bin} to \eqref{eq:bin_sys}}

The following lemma connects the min-max optimization \eqref{eq:minmax_bin} to the system of equations in \eqref{eq:bin_sys}
\begin{lem}[Uniqueness of solutions to \eqref{eq:bin_sys}]\label{lem:unique_1}
Assume that the optimization problem in \eqref{eq:minmax_bin} yields a unique and bounded solution $(\alpha>0,\mu,\upsilon>0,\gamma>0)$. Then the equations \eqref{eq:bin_sys} have a unique and bounded solution $(\alpha>0,\mu,\tau>0)$ where $\tau = \upsilon/\gamma.$
\end{lem}

\begin{proof}
By direct differentiation with respect to the variables $(\mu,\alpha,\upsilon,\gamma)$, the first order optimality conditions of the min-max optimization in \eqref{eq:minmax_bin} are as follows:
\begin{align}\label{eq:foureq_bin}
\begin{split}
\mathbb{E}\left[Sf(S)\envdx{\ell}{\ourx}{\frac{\upsilon}{\gamma}}\right] = -\la\mu,\,
\la \alpha +\mathbb{E}\left[G \envdx{\ell}{\ourx}{\frac{\upsilon}{\gamma}} \right]=\frac{\gamma}{\sqrt{\delta}},\\
\frac{1}{\gamma}\mathbb{E}\left[\envdla{\ell}{\ourx}{\frac{\upsilon}{\gamma}}\right]=-\frac{\gamma}{2},\,
-\frac{\upsilon}{\gamma^2}\mathbb{E}\left[\envdla{\ell}{\ourx}{\frac{\upsilon}{\gamma}}\right]+\frac{\upsilon}{2} =\frac{\alpha}{\sqrt{\delta}}.
\end{split}
\end{align}
Assumptions of the lemma imply that the saddle point of the optimization problem in \eqref{eq:minmax_bin} is unique and bounded, therefore \eqref{eq:foureq_bin} yields a unique bounded solution $(\alpha>0,\mu,\upsilon>0,\gamma>0)$. By denoting $\tau=\upsilon/\gamma$ and using the fact that $\envdla{\Lm}{x}{\tau} = -\frac{1}{2}(\envdx{\Lm}{x}{\tau})^2$ (as implied by \eqref{eq:mor_der1}-\eqref{eq:mor_der2}) we reach the Equations \eqref{eq:bin_sys} i.e.,
\begin{subequations}\label{eq:bin_main}
\bea
 \Exp\Big[S\,f(S) \cdot\envdx{\Lm}{\ourx}{\tau}  \Big]&=-\la\mu , \label{eq:murbin_main}\\
 {\tau^2}\,{\delta}\cdot\Exp\Big[\,\left(\envdx{\Lm}{\ourx}{\tau}\right)^2\,\Big]&=\alpha^2 ,
\label{eq:alphabin_main}\\
\tau\,\delta\cdot\E\Big[ G\cdot \envdx{\Lm}{\ourx}{\tau}  \Big]&=\alpha(1-\la\tau\delta) .
\label{eq:lambdabin_main}
\eea
\end{subequations}
The uniqueness of $(\alpha>0,\mu,\tau>0)$ as the solution to \eqref{eq:bin_main} follows from the uniqueness of the solution $(\alpha>0,\mu,\upsilon>0,\gamma>0)$ to \eqref{eq:foureq_bin}. In particular if there are two distinct solutions $(\alpha_1,\mu_1,\tau_1)$ and $(\alpha_2,\mu_2,\tau_2)$ to the Equations \eqref{eq:bin_main}, then we reach contradiction by noting that $(\alpha_1,\mu_1,\upsilon_1:=\alpha_1/\sqrt{\delta},\gamma_1:=\alpha_1/(\tau_1\sqrt{\delta}))$ and $(\alpha_2,\mu_2,\upsilon_2:=\alpha_2/\sqrt{\delta},\gamma_2:=\alpha_2/(\tau_2\sqrt{\delta}))$ are two distinct points satisfying the Equations \eqref{eq:foureq_bin}. This completes the proof of the lemma.
\end{proof}

\subsection{Completing the proof of Theorem \ref{propo:boundedness}}
We are now ready to complete the proof of Theorem \ref{propo:boundedness}. Based on Lemma \ref{lem:unique_1}, for the system of equations in \eqref{eq:bin_sys} to have a unique and bounded solution, it suffices that  $(\alpha^\star>0,\mu^\star,\upsilon^\star>0,\gamma^\star>0)$ as the solution of \eqref{eq:minmax_bin} is unique and bounded. Since $\Theta$ is convex-concave and the optimality sets are bounded from Lemma \ref{lem:theta_convexity}(a)-(e), a saddle point of $\Theta$ exists \cite[Cor.~37.3.2]{Roc70}. Additionally, based on the assumptions of the theorem and in view of Lemma \ref{lem:theta_convexity}(d),(e), $\Theta$ is jointly strictly-convex in $(\alpha,\mu,\upsilon)$ and strictly-concave in $\gamma$ which implies the uniqueness of $(\alpha^\star>0,\mu^\star,\upsilon^\star>0,\gamma^\star>0)$ as a solution to \eqref{eq:minmax_bin}. This completes the proof of the theorem. 

As mentioned in the main body of the paper, we conjecture that some of the technical conditions of Theorem \ref{propo:boundedness}, albeit mild in their current form, can be relaxed even further. Refining these conditions can be an interesting topic of future work, but is out of the scope of this paper. We mention in passing that the conclusions of Theorem \ref{propo:boundedness} also hold true if we replace the two-times differentiability condition by an assumption that the loss is one-time differentiable and strictly convex.

\section{Fundamental Limits for Linear Models: Proofs for Section \ref{sec:linear}}

\subsection{Auxiliary Results}
\begin{lem}[Boundedness of $\tau$ in \eqref{eq:eq_main}]\label{lem:tau_bound_lin}
Let $\Lm(\cdot)$ be a non-linear, convex and twice differentiable function, $\la>0$ and $\delta>0$ and the pair $(\alpha,\tau)$ be a solution to \eqref{eq:eq_main0} where $\alpha>0$. Then,
$
0 < \tau < \frac{1}{\la\delta}.
$
\end{lem}
\begin{proof}
Using Stein's lemma (aka Gaussian integration by parts) we find that 
\begin{align*}
\E\Big[ G\cdot \envdx{\Lm}{\al\,G+Z}{\tau}  \Big]&=\alpha\,\E\Big[ \envddx{\Lm}{\al\,G+Z}{\tau}  \Big]
\end{align*}
Therefore the equation in the LHS in \eqref{eq:eq_main0} is equivalent to
\begin{align}\label{eq:steinthird}
\tau\delta\,\E\Big[ \envddx{\Lm}{\al\,G+Z}{\tau}  \Big] = 1-\la\tau\delta.
\end{align}
Next we prove that under the assumptions of the lemma, $\E\Big[ \envddx{\Lm}{\ourx}{\tau}  \Big]$ is positive. First using properties of Morea-envelopes in \eqref{eq:secondderivative_mor}, we have 
\begin{align}
\E\Big[ &\envddx{\Lm}{\al\,G+Z}{\tau}  \Big] = \E\left[ \frac{\Lm''(\prox{\Lm}{\al\,G+Z}{\tau})}{1+\tau \Lm''(\prox{\Lm}{\alpha G + Z}{\tau})}\right]\ge0.\label{eq:mor_der_thirdeq}
\end{align}
In particular, we see that equality is achieved in \eqref{eq:mor_der_thirdeq} is achieved if and only if 
\begin{align*}
\forall x\in \R: \quad \envddx{\Lm}{x}{\tau} =0.
\end{align*}
Or equivalently,
\begin{align}\label{eq:mor_lin}
\exists\,c_{_1},c_{_2} \in \R : \text{s.t.} \;\; \forall x\in \R : \env{\Lm}{x}{\tau} = c_{_1}x+c_{_2}.
\end{align}
Finally, using Proposition \ref{propo:inverse} to ``invert" the Moreau envelope function, we find that the loss function $\Lm(\cdot)$ satisfying \eqref{eq:mor_lin} is such that
\begin{align*}
\forall x\in\R:\quad\Lm(x) = -\env{-c_{_1}I-c_{_2}}{x}{\tau}= c_{_1}x + \frac{\tau c_{_1}^2}{2}+c_{_2},
\end{align*}
where $I(\cdot)$ is the identity function i.e. $I(t) = t,$  $\forall t\in \R$. But according to the assumptions of the lemma, $\Lm$ is a \emph{non-linear} convex function. Thus, it must hold that $\E\left[ \envddx{\Lm}{\al\,G+Z}{\tau}  \right]>0$. Using this and the assumptions on $\la$ and $\delta$, the advertised claim follows directly from \eqref{eq:steinthird}.
\end{proof}

\subsection{Proof of Theorem \ref{thm:lowerbound_reg}}\label{sec:proof_bin_lowerbound}
Fix a convex loss function $\Lm$ and regularization parameter $\la\geq 0$. Let $(\alpha>0,\tau>0)$ be the unique solution to 
\begin{subequations}\label{eq:eq_main}
\begin{align}
 \delta \tau^2\cdot \E \Big[\Big(\envdx{\Lm}{\al\,G+Z}{\tau}\Big)^2\,\Big]&=\alpha^2- \la^2\delta^2 \tau^2, \label{eq:eq_alpha}\\
 \delta \tau \cdot\E\Big[G\cdot\envdx{\Lm}{\al\,G+Z}{\tau}\Big]&=\alpha\,(1-\la\delta\tau) \label{eq:eq_tau}.
\end{align}
\end{subequations}
For convenience, let us define the function $\Psi: \R_{\ge0} \times [0,1)\rightarrow \R$:
\bea\label{eq:psi_reg1}
\quad\Psi(a,x):=\frac{(a^2-x^2\,\delta^2)\,\Ic(V_a)}{(1-x\,\delta)^2}.
\eea
Then,  $\alpha_{\star}>0$ as in \eqref{eq:alphaopt_thm} is equivalently expressed as
\begin{align}\label{eq:psi_reg}
\alpha_{\star}:= \min_{\substack{0\le x<1/\delta}} \left\{a \geq 0: \; \Psi(a,x)= \frac{1}{\delta} \right\}.
\end{align}

Before everything, let us show that $\alpha_\star$ is well-defined, i.e., that the feasible set of the minimization in \eqref{eq:psi_reg} is non-empty for all $\delta>0$ and random variables $Z$ {satisfying Assumption \ref{ass:noise}}.  Specifically, we will show that there exists $a\geq 0$ such that $\Psi\left(a,\frac{a}{(1+a) \delta}\right)=1/\delta$. It suffices to prove that the range of the function $\Psit(a):=\Psi\left(a,\frac{a}{(1+a)\delta}\right)$ is $(0,\infty)$. Clearly, the function $\Psit$ is continuous in $\R_{\ge0}$. Moreover, it can be checked that $\Psit(a)=(a^2+2a)\Psi_0(a)$ where $\Psi_0(a):=a^2\Ic(V_a)$. By Lemma \ref{lem:Fisher_lim},  $\lim_{a\rightarrow 0}\Psi_0(a) = 0$  and $\lim_{a\rightarrow +\infty}\Psi_0(a) = 1$.  Hence, we find that $\lim_{a\rightarrow0}\Psit(a) = 0$ and $\lim_{a\rightarrow+\infty}\Psit(a) = +\infty$, as desired.

We are now ready to prove the main claim of the theorem, i.e.,
\bea\label{eq:desired21}
\alpha \geq \alpha_\star.
\eea
Denote by $\phi_\alpha$ the density of the Gaussian random variable $\alpha G$. We start with the following calculation:
\begin{align}
\E&\bigg[ G\cdot \envdx{\Lm}{V_\alpha }{\tau}  \bigg]  = -\alpha\iint\envdx{\Lm}{ u+z}{\tau}  \phi^\prime_{\alpha}(u)p_{Z}(z)\mathrm{d}u\mathrm{d}z  \nn\\
&= -\alpha\iint\envdx{\Lm}{ v}{\tau}  \phi^\prime_{\alpha}(u)p_{Z}(v-u)\mathrm{d}u\mathrm{d}v  \nn \\
&=  -\alpha\int\envdx{\Lm}{ v}{\tau}  p^\prime_{_{V}}(v)\mathrm{d}v =  -\alpha\,\E\Big[\envdx{\Lm}{V _\alpha}{\tau}\cdot\ksi_{{V_{_{\alpha}}}}(V_\alpha)\Big], \label{eq:equivalence}
\end{align}
where for a random variable $V$, we denote its score function with $\ksi_V(v) := p'_V(v)/p_V(v)$ for $v\in\R$.
Using \eqref{eq:equivalence} and $\alpha>0$,  \eqref{eq:eq_tau} can be equivalently written as following, 
\bea\label{eq:ksi}
1-\la\,\delta\,\tau = -\delta\tau\cdot\E\Big[\envdx{\Lm}{V_\alpha}{\tau}\cdot\ksi_{{V_{_{\alpha}}}}(V_\alpha)\Big].
\eea
Next, by applying Cauchy-Shwarz inequality, recalling $\E[(\ksi_{{V_{_{\alpha}}}}(V_\al))^2]=\Ic\left(V_\al\right)$ and using \eqref{eq:eq_alpha}, we have that
\begin{align*}
\left(\E\Big[\envdx{\Lm}{ V_\al}{\tau}\cdot\ksi_{{V_{_{\alpha}}}}(V_\al)\Big]\right)^2&\le \E \Big[\Big(\envdx{\Lm}{V_\al}{\tau}\Big)^2\,\Big]\cdot\Ic\left(V_\al\right) = \frac{(\alpha^2-\la^2\,\delta^2\,\tau^2)\,\Ic(V_\al)}{\delta \tau^2},
\end{align*}
where we have also used the fact that $\tau>0$.
To continue, we use \eqref{eq:ksi} to rewrite the LHS above and deduce that:
\bea
\left(\frac{1-\la\,\delta\,\tau}{\delta\tau}\right)^2\le \frac{(\alpha^2-\la^2\,\delta^2\,\tau^2)\,\Ic(V_\al)}{\delta \tau^2}.
\eea
By simplifying the resulting expressions we have proved that $(\alpha,\tau)$ satisfy the following inequality: 

\bea\label{eq:inequality}
\frac{(\alpha^2-\la^2\,\delta^2\,\tau^2)\,\Ic(V_\al)}{(1-\la\,\delta\,\tau)^2} \ge \frac{1}{\delta}.
\eea

In the remaining, we use \eqref{eq:inequality} to prove \eqref{eq:desired21}. For the sake of contradiction to  \eqref{eq:desired21}, assume that there exists a valid triplet $(\alpha,\la,\tau)$ such that $\alpha< \alpha_\star$. Recall by inequality \eqref{eq:inequality} that $\alpha$ satisfies:
\begin{align}\label{eq:psi_ineq}
\Psi\Big({\alpha}\,, \,{\la\,\tau}\Big) \ge \frac{1}{\delta}.
\end{align}
We show first that \eqref{eq:psi_ineq} holds with strict inequality. To see this, suppose that $\Psi({\alpha}, \la\,\tau) =1/\delta$. 
From Lemma \ref{lem:tau_bound_lin}, it also holds that $\la\,\tau\in(0,1/\delta)$. Hence, the pair $(\alpha,\la\,\tau)$ is a feasible point in the minimization in \eqref{eq:psi_reg}. Combining this with optimality of $\alpha_\star$ lead to the conclusion that $\alpha_\star\geq \alpha$, which contradicts our assumption $\alpha< \alpha_\star$. 
%
Therefore we consider only the case where \eqref{eq:psi_ineq} holds with strict inequality i.e., $\Psi({\alpha},\,\la\tau)>1/\delta$. 

To proceed, note that $\Psi(0,x)\le0$ for all $x\in[0,1).$ Thus, by continuity of the function $a\mapsto\Psi(a,x)$ for fixed $x\in[0,1/\delta)$:
\begin{align}\label{eq:alphatilde}
\exists\,\widetilde{\alpha} : \;\text{s.t.} \;\;0\le\widetilde{\alpha}<{\alpha},\;\;\text{and}\;\; \Psi\,\Big(\widetilde{\alpha}\,,\,{\la}{\,\tau}\Big) = \frac{1}{\delta}.
\end{align}
By recalling our assumption that ${\alpha}<\alpha_{\star}$, we can deduce that \eqref{eq:alphatilde} in fact holds for $\widetilde{\alpha}<\alpha_{\star}$. However, this is in contradiction with the optimality of $\alpha_{\star}$ defined in \eqref{eq:psi_reg}. This shows that for all achievable $\alpha$ it must hold that $\alpha\ge\alpha_{\star}$. This proves the claim in \eqref{eq:desired21} and completes the proof of the theorem.

\subsection{Proof of Lemma \ref{thm:opt_reg}}
To prove the claim of the lemma, it suffices to show that the proposed loss function and regularization parameter, satisfy the system of equations in \eqref{eq:eq_main} with $\alpha = \alpha_{\star}$. For this purpose we show that $(\Lm, \la, \alpha, \tau) = (\Lm_{\star}, \la_{\star}, \alpha_{\star}, 1)$ satisfy \eqref{eq:eq_main}. 

  First, we recognize that for the candidate optimal loss function in Lemma \ref{thm:opt_bin} we have $\forall v\in \R$ that 
\bea \label{eq:envdx_opt}
\envdx{\Lm_{\star}}{v}{1} = -\frac{\alpha_{\star}^2 - \la_{\star}^2\,\delta^2}{1-\la_{\star}\delta}\cdot \ksi_{{V_{\star}}}(v).
\eea
Thus by replacing the proposed parameters in \eqref{eq:eq_alpha} we have :
\begin{align*}
\delta\,\E \left[\Big(\envdx{\Lm_{\star}}{V_{\star}}{1}\Big)^2\,\right] &= \delta \left(\frac{\alpha_{\star}^2-\la_{\star}^2\,\delta^2}{1- \la_{\star}\,\delta}\right)^2\Ic\left(V_{\star}\right) = \alpha_{\star}^2-\la_{\star}^2\delta^2,
\end{align*}
where for the last line we used the definitions of $\alpha_{\star}$ and $\la_{\star}$ in the statement of the lemma. This proves the claim for \eqref{eq:eq_alpha}. To show that Equation \eqref{eq:eq_tau} is satisfied we use its equivalent expression in \eqref{eq:ksi} and also replace \eqref{eq:envdx_opt} in \eqref{eq:ksi}. Specifically, this shows that
\begin{align*}
\delta\,\E\bigg[ G\cdot \envdx{\Lm_{\star}}{V_{\star}}{1}  \bigg] &=-\delta\,\alpha_{\star}\,\E\bigg[\envdx{\Lm_{\star}}{ V_{\star}}{1}\cdot\ksi{_{V_\star}}(V_{\star})\bigg] \\
&= \frac{ \delta\,\alpha_{\star}\,(\alpha_{\star}^2-\la_{\star}^2\,\delta^2)\cdot\Ic(V_{\star})}{1-\la_{\star}\,\delta} = \alpha_{\star}(1-\la_{\star}\,\delta),
\end{align*}
from which we conclude that Equation \eqref{eq:eq_tau} is satisfied. This completes the proof of the lemma.

\subsection{Proof of Lemma \ref{cor:LS_reg}}
By letting $\Lm(t) = t^2$ we find that $\env{\Lm}{x}{\tau} = \frac{x^2}{2\tau+1}$ for all $x\in \R$ and $\tau\in\R_{>0}$. Using this in Equations \eqref{eq:eq_main} and a after a few algebraic simplifications we arrive at the following closed-form expression for $\alpha_{{\ell_{_2},\la}}^2$ for all $\la\ge0$ and random variables $Z$ with finite second moment,
\bea\label{eq:alphaLSreg}
\alpha_{{\ell_{_2},\la}}^2 = \frac{1}{2}\left(1-\E[Z^2]-\delta\right)  + \frac{\E[Z^2](\la+2\delta+2)+2(\delta-1)^2+\la(\delta+1)}{2\sqrt{(\la+2\delta-2)^2+8\la}}.
\eea
Next, by using direct differentiation to optimize this over $\la\ge0$, we derive $\la_{\rm opt} = 2\E[Z^2]$ and the resulting expression for $\alpha_{{\ell{_2},\la_{\rm opt}}}^2$ in the statement of the lemma.

\subsection{Proof of Corollary \ref{cor:lowerbound}}\label{sec:proofofcor_lin}
As mentioned in the main body of the paper, the difficulty in deriving a closed-form expression for $\alpha_{\star}$ in \eqref{eq:alphaopt_thm} is due to the fact that in general $\Ic(V_a) = \Ic(a G+Z)$ may not be expressible in closed-form with respect to $a$. The core idea behind this corollary is using Stam's inequality (see Proposition \ref{propo:Fisher}) to bound $\Ic(V_a)$ in terms of $\Ic(a G) = a^{-2}$ and $\Ic(Z)$.
Specifically, applying \eqref{eq:Fisher_Stam} to the random variables $a\,G$ and $Z$ we find that:
\bea\label{eq:stam}
\Ic(V_a) = \Ic(a\, G+Z)\le\frac{\Ic(Z)}{1+a^2\Ic(Z)}.
\eea
Substituting the RHS above in place of $\Ic(V_a)$ in  the definition of $\alpha_\star$ in \eqref{eq:alphaopt_thm}, let us define
 $\widehat{\al}$ as follows:
\bea\label{eq:alphahat}
\widehat{\al} := \min_{0\le x<1/\delta}\left\{a\,\ge0:\frac{(a^2-x^2\,\delta^2)\,\Ic(Z)}{(1-x\,\delta)^2(1+a^2\Ic(Z))}\ge\frac{1}{\delta}\right\}.
\eea

The remaining of proof has two main steps. First, we show that 
\bea\label{eq:Cor21_step1}
\alpha_\star^2\ge\widehat{\al}^2.
\eea
Second, we solve the minimization in \eqref{eq:alphahat} to yield a closed-form expression for $\widehat{\al}$. 

Towards proving \eqref{eq:Cor21_step1}, note from the definition of $\alpha_\star$ and inequality \eqref{eq:stam} that there exists $x_\star\in[0,1/\delta)$ such that
%
\bea\nn
\frac{1}{\delta} = \frac{(\al_\star^2-x_\star^2\,\delta^2)\,\Ic(V_{\star})}{(1-x_\star\,\delta)^2} \leq  \frac{(\al_\star^2-x_\star^2\,\delta^2)\,\Ic(Z)}{(1-x_\star\,\delta)^2(1+\al_\star^2\,\Ic(Z))}.
\eea
Thus, the pair $(\al_\star,x_\star)$ is feasible in \eqref{eq:alphahat}. This and optimality of $\widehat{\al}$ in \eqref{eq:alphahat} lead to \eqref{eq:Cor21_step1}, as desired.

The next step is finding a closed-form expression for $\widehat{\alpha}$. Based on \eqref{eq:alphahat} and few algebraic simplifications we have :
\bea
\widehat{\al}^2 &= \min_{0\le \,x\,<1/\delta} \left\{ a^2 : a^2\,\Ic(Z)\cdot(\delta-(1-x\,\delta)^2) \ge(1-x\,\delta)^2 +\delta^3 x^2\Ic(Z)\right\}\nn\\[5pt]
&= \min_{\max\{0, \frac{1-\sqrt{\delta}}{\delta}\}\le \,x \,< 1/\delta}  \left\{a^2 : a^2 \ge \frac{(1-x\delta)^2+\delta^3\,x^2\,\Ic(Z)}{\Ic(Z)\cdot(\delta-(1-x\delta)^2)}\right\}\nn\\[5pt]
&= \min_{\max\{0, \frac{1-\sqrt{\delta}}{\delta}\}\le \,x \,< 1/\delta} \, \left\{\frac{(1-x\delta)^2+\delta^3\,x^2\,\Ic(Z)}{\Ic(Z)\cdot(\delta-(1-x\delta)^2)}\right\}.\label{eq:alphahat2}
\eea
The last equality above is true because the fraction in the constraint in the second line is independent of $a$. Next, by minimizing with respect to the variable $x$ in \eqref{eq:alphahat2}, we reach $\widehat{\alpha}^2 = h_\delta(1/\Ic(Z))$.

Finally, we know from Proposition \ref{propo:Fisher}(f) that equality in \eqref{eq:stam} is achieved if and only if the noise is Gaussian i.e. $Z\sim\mathcal{N}(0,\zeta^2)$ for some $\zeta>0$. Thus, if this is indeed the case, then $\al_\star=\widehat{\al}$ and the lower bound is achieved with replacing the Fisher information of $Z$ i.e, $\Ic(Z) =\zeta^{-2}$. This completes the proof of the corollary.
%
\subsection{Proof of Equation \eqref{eq:omega_LB}}\label{sec:proof_omega_LB}
First, we prove the bound $\omega_{_\delta} \geq \left(\Ic(Z)\, \E[Z^2]\right)^{-1}.$ Fix $\delta>0$ and consider the function $\widetilde{h}_{\delta}(x):=h_{\delta}(x)/x$ for $x\geq 0$. Direct differentiation and some algebra steps suffice to show that $\widetilde{h}_{\delta}(x)$ is decreasing. Using this and the fact that $1/\Ic(Z) \le \E[Z^2]$ (cf. Proposition \ref{propo:Fisher} (c)), we conclude with the desired.


Next, we prove the lower bound $\omega_\delta\geq 1-\delta$. Fix any $\delta>0$. First, it is straightforward to compute that
$
h_{\delta}(0) = \max\{1-\delta,0\} \geq 1-\delta.
$
Also, simple algebra shows that $h_{\delta}(x)\leq 1, x\geq 0$. From these two and the increasing nature of $h_{\delta}(x)$ we conclude that  $1-\delta \le h_{\delta}(x) \le 1$, for all $x \geq 0$. The desired lower bound follows immediately by applying these bounds to the definition of $\omega_\delta$.


\section{Fundametal Limits for Binary Models: Proofs for Section \ref{sec:binary}}
\subsection{Discussion on Assumption \ref{ass:label}}\label{sec:sfs}
As per Assumption \ref{ass:label}, the link function must satisfy $\E[Sf(S)]\neq 0$. This is a rather mild assumption in our setting. For example, it is straightforward to show that it is satisfied for the Signed, Logistic and Probit models. More generally, for a link function $f:\R\rightarrow\{\pm 1\}$ and $S\sim\mathcal{N}(0,1)$, the probability density of $Sf(S)$ can be computed as follows for any $x\in \R$:
\begin{align}\label{eq:sfs}
p_{_{Sf(S)}}(x) = \Big(1+ \widehat{f}(x)-\widehat{f}(-x)\Big) \frac{\exp(-x^2/2)}{\sqrt{2\pi}}, \quad\quad \widehat{f}(x):= \Pro\left(f(x)=1\right).
\end{align}
From this and the fact that $\exp(-x^2/2)$ is an even function of $x$, we can conclude that Assumption \ref{ass:label} is valid if $\widehat{f}(x)$ is monotonic and non-constant based on $x$ (e.g., as in the Signed, Logistic and Probit models) . In contrast, Assumption \ref{ass:label} fails if the function $\widehat{f}$ is even. Finally, we remark that using \eqref{eq:sfs}, it can be checked that $S\,f(S)\sim \mathcal{N}(\mu,\zeta^2)$ if and only if $(\mu,\zeta)=(0,1)$, and consequently only if $\widehat{f}$ is an even function. Based on these, we conclude that for all link functions $f$ satisfying Assumption \ref{ass:label}, the resulting distribution of $Sf(S)$ is non-Gaussian. 
Finally, we remark that $\nu_f=\E[Sf(S)]$ is the first Hermite coefficient of the function $f$ and the requirement $\nu_f\neq 0$ arises in a series of recent works on high-dimensional single-index models, e.g.,  \cite{Ver,genzel2016high}; see also \cite{mondelli2017fundamental,lu2017phase} for algorithms specializing to scenarios in which $\nu_f=0$.

\subsection{Discussion on the Classification Error \eqref{eq:testerror2}}\label{sec:proofoferr}
First, we prove that for an estimator $\widehat\w_{\Lm,\la}$, the relation $\Pro(\sig_{\Lm,\la}G+Sf(S)<0)$ determines the high-dimensional limit of classification error. Then we show that the classification error is indeed an increasing function of $\sig_{\Lm,\la}$ for most well-known binary models.

For the estimator $\wh_{{\Lm,\la}}$ obtained from \eqref{eq:opt_bin_main}, and $\x_0$ denoting the true vector with unit norm, the parameters $\mu_{{\Lm,\la}}$ and $\al_{{\Lm,\la}}$ denote the high-dimensional terms of bias and variance,
\bea
\x_0^T  \wh_{{\Lm,\la}} \rP\,\mu_{{\Lm,\la}},\label{eq:mu} \\[5pt]
 \left\|\wh_{{\Lm,\la}}- \mu_{\Lm,\la} \,\x_0\right\|_{2}^{2} \rP \,\alpha_{{\Lm,\la}}^2.\label{eq:error_bin}
\eea
We note that by rotational invariance of Gaussian distribution we may assume without loss of generality that $\x_0 = \left[1,\,0,\,0,\,\cdots,\,0\right]^T \, \in \R^n$. Therefore we deduce from \eqref{eq:mu} and \eqref{eq:error_bin} that 
\begin{align*}
\wh_{{\Lm,\la}}(1) \rP \mu_{\Lm,\la}, \quad\quad \sum_{i=2}^n \left(\wh_{{\Lm,\la}}(i)\right)^2 \rP \alpha_{\Lm,\la}^2.
\end{align*}
Using these, we derive the following for the classification error :
\begin{align*}
\mathcal{E}_{{\Lm,\la}} &=  \Pro \left(\;f\left(\ab^T \x_0\right) \; \ab^T\,\wh_{{\Lm,\la}} < 0 \;\right)\\[5pt]
&= \Pro \left( \;f\left(\ab(1)\right) \cdot \Big(\wh_{{\Lm,\la}}(1) \ab(1) + \wh_{{\Lm,\la}} (2) \ab(2) + \cdots + \wh_{{\Lm,\la}}(n) \ab(n)\Big)<0\;\right).
\end{align*}
Recalling Assumption \ref{ass:gaussian} we have $\ab \sim \mathcal{N}\left(\mathbf{0},\,\mathbf{I}\right)$. Thus by denoting $S,G\simiid \mathcal{N} (0,1)$ and assuming without loss of generality that $\mu_{\Lm,\la}>0$, we derive \eqref{eq:testerror2}.

Next, we show that for the studied binary models in this paper, the high-dimensional limit for the classification error is increasing based on effective error term $\sig>0$. In particular, we find that if $p_{_{Sf(S)}}(x)>p_{_{Sf(S)}}(-x)$ for $x\in\R_{>0}$ then it is guaranteed that $a\mapsto\Pro(aG+Sf(S)<0)$ is an increasing function for $a>0$. To show this, we denote by $\phi$ the density of standard normal distribution and let $a_1>a_2$ to be two positive constants, then under the given condition on $p_{_{Sf(S)}}$, we deduce that,
\begin{align*}
\Pro\left(Sf(S)<a_1G\right)&\,-\,\Pro\left(Sf(S)<a_2G\right) = \\
&\int_0^{+\infty}\int_{a_2g}^{a_1g}p_{_{Sf(S)}}(x)\,\phi(g)\;\text{d}x\,\text{d}g -  \int_{-\infty}^{0}\int_{a_1g}^{a_2g}p_{_{Sf(S)}}(x)\,\phi(g)\;\text{d}x\,\text{d}g >0.
\end{align*}
This shows the desired. Importantly, we remark that in view of \eqref{eq:sfs}, this condition on the density of $Sf(S)$ is satisfied for many well-known binary models including Logistic, Probit and Signed.

\subsection {Proof of Theorem \ref{thm:lowerbound_bin}}\label{sec:proofofbin}
We  need the following auxiliary result, which we prove first.

\begin{lem}[Boundedness of $\tau$ in \eqref{eq:bin_main}]\label{lem:boundedtau}
Fix $\delta>0$ and $\la>0$ and let $\Lm$ be a convex, twice differentiable and non-linear function. Then all solutions $\tau$ of the system of equations in \eqref{eq:bin_main} satisfy $0<\tau<\frac{1}{\la\delta}$.
\end{lem}
\begin{proof}
The proof follows directly from the proof of Lemma \ref{lem:tau_bound_lin} by replacing $Z$ with $\mu\,Sf(S)$. Note that the Equation \eqref{eq:lambdabin_main} can be obtained by replacing $Z$ with $\mu Sf(S)$ in Equation \eqref{eq:eq_tau}. \end{proof}


Next, we proceed to the proof main of Theorem \ref{thm:lowerbound_bin}. For convenience, let us define the function $\Phi:\R_{\ge0}\times[0,1/\delta)\rightarrow\R$ as following : 
\begin{align}
\Phi(s,x)&:= \frac{1-s^2(1-s^2\Ic(W_s))}{\delta s^2(s^2\Ic(W_s)+\Ic(W_s)-1)}-2x + x^2\delta(1+s^{-2}). \label{eq:phi_bin}
\end{align}
Then, $\sigma_\star$ as in \eqref{eq:sigopt_thm} is equivalently expressed as:
\bea
\sig_{\star}&:= \min_{\substack{0\le x<1/\delta}} \left\{ s \ge 0: \; \Phi(s,x)= 1\right\},\label{eq:sig_opt_phi_bin}
\eea

Before everything, we show that $\sig_{\star}$ is well defined, i.e., the feasible set of the minimization in \eqref{eq:sig_opt_phi_bin} is non-empty for all $\delta>0$ and link functions $f(\cdot)$ satisfying Assumption \ref{ass:label}. Specifically, we will show that for any $\delta>0$ there exists $s\geq 0$ such that $\Phit:=\Phi\big(s,\frac{s}{\delta(1+s)}\big) = 1$. It suffices to prove that the range of the function $\Phit$ is $(0,\infty)$. Clearly, the function is continuous in $\R_{\geq 0}$. Moreover, it can be checked that
\bea\label{eq:Phit}
\Phit(s) = \Phi_0(s)  + \frac{2}{\delta(1+s^2)},\quad \text{where}\quad \Phi_0(s) := \frac{1-s^2\Ic(W_s)}{\delta\,s^2(s^2\Ic(W_s)+\Ic(W_s)-1)}.
\eea
But, by Lemma \ref{lem:Fisher_lim},  $\lim_{s\rightarrow0}\;\;s^2\,\Ic(W_s) = 0$  and $\lim_{s\rightarrow+\infty}s^2\,\Ic(W_s) = 1$. Using these, we can show that $\lim_{s\rightarrow0}\;\;\Phi_0(s) = +\infty$ and $\lim_{s\rightarrow+\infty}\;\;\ \Phi_0(s) = 0$. Combined with \eqref{eq:Phit}, we find that $\lim_{s\rightarrow0}\;\;\Phi_0(s) = +\infty$ and $\lim_{s\rightarrow+\infty}\;\;s^2\,\Ic(W_s) = 0$. This concludes the proof of feasibility of the minimization in \eqref{eq:phi_bin}.

We are now ready to prove the main claim of the theorem. Fix convex loss function $\Lm$ and regularization parameter $\la\geq$. Let $(\alpha>0,\mu, \tau>0)$ be the unique solution to \eqref{eq:bin_main} and denote $\sigma=\alpha/\mu$. We will prove that 
\bea\label{eq:WTS_bin}
\sigma \geq \sigma_\star.
\eea

The first step in the proof will be to transform the equations \eqref{eq:bin_main} in a more appropriate form. In order to motivate the transformation, note that the performance of the optimization problem in \eqref{eq:opt_bin_main} is unique up to rescaling. In particular consider the following variant of the optimization problem in \eqref{eq:opt_bin_main} : 
\bea
{\widehat{\vb}}_{{\Lm,\la}} : = \arg\min_{\w}\left[\frac{c_{_1}}{m}\sum_{i=1}^m {\Lm}\left(c_{_2}\,y_i\ab_i^T \w\right)+c_{_1} \la \|c_{_2}\w\|^{2}\right],\quad c_{_1}>0,\,c_{_2}\neq0. \nn
\eea
It is straightforward to see that,  regardless of the values of $c_{_1}$ and $c_{_2}$, $
\corr{{\widehat{\w}}_{{\Lm,\la}}}{\x_0} = \corr{\widehat{\vb}_{{\Lm,\la}}}{\x_0}$, where recall that $\widehat{\w}_{{\Lm,\la}}$ solves \eqref{eq:opt_bin_main}. Thus in view of \eqref{eq:corr_lim}, we see that the error $\sig$ resulting from $\widehat{\w}_{{\Lm,\la}}$ and $\widehat{\vb}_{{\Lm,\la}}$ are the same. Motivated by this observation, we consider the following rescaling for the loss function and regularization parameter:
\bea
\widetilde{\Lm}(\cdot):=\frac{\tau}{\mu^2}\,\Lm(\mu\,\cdot),\quad\quad \widetilde{\la}:=\tau\la, \label{eq:ltilde,ltilde}
\eea
From standard properties of Moreau-envelope functions it can be shown that 
\bea\nn
\envdx{\Lmt}{\cdot/\mu\,}{1} = \frac{\tau}{\mu}\,\envdx{\Lm}{\cdot\,}{\tau}.
\eea
Using these transformations, we can rewrite the system of equations \eqref{eq:bin_main} in terms of $\sigma$, $\Lmt$ and $\lat$ as follows:
\begin{subequations}\label{eq:bin_main2}
\bea
 \Exp\Big[Sf(S) \cdot\envdx{\Lmt}{W_\sig}{1}  \Big]&=-\lat , \label{eq:murbin_main2}\\
\Exp\Big[\,\Big(\envdx{\Lmt}{W_\sig}{1}\Big)^2\,\Big]&=\sig^2/\delta ,
\label{eq:alphabin_main2}\\
\E\Big[ G\cdot \envdx{\Lmt}{W_\sig}{1}  \Big]&=\sig(1-\lat\delta)/\delta .
\label{eq:lambdabin_main2}
\eea
\end{subequations}
where we denote $W_\sig:=\sig G+Sf(S).$ 

Next, we further simplify \eqref{eq:bin_main2} as follows. Similar to the procedure leading to \eqref{eq:equivalence}, here also we may deduce that,
$$
\E\Big[ G\cdot \envdx{\Lmt}{W_\sig}{1}  \Big] = -\sig\,\E\Big[ \ksi_{{W_\sig}}(W_\sig)\cdot \envdx{\Lmt}{W_\sig}{1}  \Big].
$$
Thus \eqref{eq:lambdabin_main2} can be rewritten as 
\bea\label{eq:third_eq}
\E\Big[ \ksi_{_{W_\sig}}(W_\sig)\cdot \envdx{\Lmt}{W_\sig}{1}  \Big] = (\lat\delta-1)/\delta.
\eea
Additionally, we linearly combine \eqref{eq:murbin_main2} and \eqref{eq:lambdabin_main2} (with coefficient $\sig$) to yield : 
\bea\label{eq:first_eq}
\Exp\Big[W_\sig \cdot\envdx{\Lmt}{W_\sig}{1}  \Big]&= \sig^2/\delta - \sig^2\lat-\lat,
\eea
Putting together \eqref{eq:alphabin_main2}, \eqref{eq:third_eq} and \eqref{eq:first_eq}, we have shown that $\sigma$ satisfies the following system of equations:
\begin{subequations}\label{eq:bin_sigma}
\bea
 \Exp\Big[W_\sig \cdot\envdx{\Lmt}{W_\sig}{1}  \Big]&= \frac{\sig^2}{\delta} - \sig^2\lat-\lat , \label{eq:mubin_sigma}\\
 \Exp\Big[\,\Big(\envdx{\Lmt}{W_\sig}{1}\Big)^2\,\Big]&=\frac{\sig^2}{\delta} ,
\label{eq:alphabin_sigma}\\
\E\Big[ \ksi_{{W_\sig}}(W_\sig)\cdot \envdx{\Lmt}{W_\sig}{1}  \Big]&=\lat-\frac{1}{\delta}.
\label{eq:lambdabin_sigma}
\eea
\end{subequations}

 Next, we will use this fact to derive a lower bound on $\sig$. To this end, 
let $\beta_{_1},\beta_{_2} \in \R$ be two real constants. By combining \eqref{eq:mubin_sigma} and \eqref{eq:lambdabin_sigma} we find that 
\bea\label{eq:befCS}
\E\Big[(\beta_{_1}W_\sig+\beta_{_2}\ksi_{{W_\sig}}(W_\sig))\cdot\envdx{\Lmt}{W_\sig}{1}\Big] = \beta_{_1}(\frac{\sig^2}{\delta}-\sig^2\lat-\lat) + \beta_{_2}(\lat-\frac{1}{\delta}).
\eea
Applying Cauchy-Schwarz inequality to the LHS of \eqref{eq:befCS} gives :
\bea
\left(\beta_{_1}\left(\frac{\sig^2}{\delta}-\sig^2\lat-\lat\right) + \beta_{_2}(\lat-\frac{1}{\delta})\right)^2&\le \E \bigg[\left(\beta_{_1}W_\sig+\beta_{_2}\ksi_{{W_\sig}}(W_\sig))\right)^2\bigg]\cdot \E\bigg[\Big(\envdx{\Lmt}{W_\sig}{1}\Big)^2\bigg]\nn\\
&= \E \bigg[\left(\beta_{_1}W_\sig+\beta_{_2}\ksi_{{W_\sig}}(W_\sig))\right)^2\bigg]\frac{\sig^2}{\delta},\label{eq:RHS}
\eea
where we used \eqref{eq:alphabin_sigma} in the last line. To simplify the expectation in the RHS of \eqref{eq:RHS}, we use the facts that $\E[W_\sig^2] = \sig^2+1$ and $\E[(\ksi_{{W_\sig}}(W_\sig))^2] = \Ic(W_\sig)$. Also by integration by parts one can derive that $\E[W_\sig\cdot\ksi_{_{W_\sig}}(W_\sig)] = -1$. Thus we arrive at the following inequality from \eqref{eq:RHS}:
\bea\label{eq:afterCS}
\left(\beta_{_1}\left(\sig^2/\delta-\sig^2\lat-\lat\right) + \beta_{_2}(\lat-1/\delta)\right)^2\le \beta_{_1}^2\,(\sig^2+1) + \beta_{_2}^2\,\Ic(W_\sig) - 2\beta_{_1}\beta_{_2}. 
\eea

Now, we  choose the coefficients $\beta_{_1}$ and $\beta_{_2}$ as follows: $\beta_{_1}=1-\lat\delta-(\sig^2-\sig^2\lat\delta-\lat\delta)\,\Ic(W_\sig)$ and $\beta_{_2}=1$. (We show later in Theorem \ref{thm:opt_bin}, that this choice lead to an achievable lower bound). Substituting these values in \eqref{eq:afterCS} and simplifying the resulting expressions yield the following inequality for $\sigma$:
\bea\label{eq:ineq_bin}
\frac{1-\sig^2(1-\sig^2\Ic(W_\sig))}{\delta\sig^2(\sig^2\Ic(W_\sig)+\Ic(W_\sig)-1)}-2\lat+ \lat^2\delta(1+\sig^{-2})\le1.
\eea

We will now finish the proof of the theorem by using \eqref{eq:ineq_bin} to prove \eqref{eq:WTS_bin}. 
%
For the sake of contradiction to \eqref{eq:WTS_bin}, assume that $\sigma<\sigma_\star$. From \eqref{eq:ineq_bin} and the notation introduced in \eqref{eq:phi_bin}, we have shown that $
\Phi\,(\sig,\widetilde{\la}) \le 1$. Recall from \eqref{eq:ltilde,ltilde} that $\widetilde{\la}=\la\tau$. But, from Lemma \ref{lem:boundedtau} it holds that $\widetilde{\la}=\la\tau < \frac{1}{\delta}$. Therefore, the pair $(\sig,\widetilde{\la})$ is feasible in the minimization problem in \eqref{eq:sig_opt_phi_bin}. By this, optimality of $\sig_{\star}$ and our assumption that $\sigma<\sigma_\star$ in \eqref{eq:sig_opt_phi_bin} it must hold that $\Phi\,({\sig},\widetilde{\la}) < 1.$
But then, since $\lim_{s\rightarrow0}\Phi\,(s,\widetilde{\la})=+\infty$ and by continuity of the function $\Phi(\cdot,x)$ for all fixed $x\in[0,1/\delta)$, we have:
\begin{align}\label{eq:alphatilde_bin}
\exists\,\sig_1 : \;\text{s.t.} \;0<\sig_1<{\sig},\;\;\text{and}\;\; \Phi(\sig_1,\widetilde{\la}) = 1.
\end{align}
Therefore $\Phi(\sig_1,\widetilde{\la}) = 1$ for $\sig_1<\sig_{\star}$, which contradicts the optimality of $\sig_{\star}$ in \eqref{eq:sig_opt_phi_bin} and completes the proof.
\subsection{Proof of Lemma \ref{thm:opt_bin}}
To prove the claim of the lemma we show that the proposed candidate-optimal loss and regularization parameter pair $(\Lm_{\star}, \la_{\star})$ satisfies the system of equations in \eqref{eq:bin_main} with $(\alpha,\mu,\tau) = (\sig_{\star},1,1)$. In line with the proof of Theorem \ref{thm:lowerbound_bin} and the equivalent representation of \eqref{eq:bin_sigma} for the equations in \eqref{eq:bin_main}, we show that $(\Lm_{\star},\la_{\star})$ satisfy all three equations in \eqref{eq:bin_sigma} with $(\sig,\mu,\tau) = (\sig_{\star},1,1).$ We emphasize that since $\mu=\tau=1$, based on \eqref{eq:ltilde,ltilde} the $\Lm_{\star}$ and $\la_{\star}$ remain the same under these changes of parameters thus $(\widetilde{\Lm}_{\star}(\cdot), \widetilde{\la}_{\star}) = (\Lm_{\star}, \la_{\star}) $.

Note that we need $\mathcal{M}_{\Lm}(\cdot)$ to be able to assess the equations in \eqref{eq:bin_sigma}. For this purpose we use inverse properties of Moreau-envelope functions in Proposition \ref{propo:inverse} to derive the following from the definition of $\Lm_{\star}$ in  \eqref{eq:optimalloss_thm} :
\bea\nn
\env{\Lm_{\star}}{w}{1} = -\frac{\eta(\la_{\star}\delta-1)}{\delta(\eta-\Ic(W_{\star}))} Q(w) -\frac{\la_{\star}\delta-1}{\delta(\eta-\Ic(W_{\star}))}\log\left(p_{_{W_{\star}}}(w)\right).
\eea
Thus, 
\bea\nn
\envdx{\Lm_{\star}}{w}{1}  = -\frac{\eta(\la_{\star}\delta-1)}{\delta(\eta-\Ic(W_{\star}))}  w -\frac{\la_{\star}\delta-1}{\delta(\eta-\Ic(W_{\star}))}\ksi_{_{W_{\star}}}(w).
\eea
Using this and the fact that $\E[W_{\star}\cdot\ksi_{_{W_{\star}}}(W_{\star})] = -1$ (derived by integration by parts), the LHS of the equation \eqref{eq:mubin_sigma} changes to
\begin{align*}
 \Exp\Big[W_{\star} \cdot\envdx{\Lm_{\star}}{W_{\star}}{1}  \Big] &= -\frac{\eta(\la_{\star}\delta-1)}{\delta(\eta-\Ic(W_{\star}))} \E\left[W_{\star}^2 \right] - \frac{\la_{\star}\delta-1}{\delta(\eta-\Ic(W_{\star}))}\E\left[W_{\star}\cdot\ksi_{_{W_{\star}}}(W_{\star})\right]\\[5pt]
 &= -\frac{\eta(\la_{\star}\delta-1)}{\delta(\eta-\Ic(W_{\star}))}(\sig_{\star}^2+1)+ \frac{\la_{\star}\delta-1}{\delta(\eta-\Ic(W_{\star}))} = \frac{\sig_{\star}^2}{\delta} - \sig_{\star}^2\la_{\star}-\la_{\star},
\end{align*}
where for the last step, we replaced $\eta$ according to the statement of the lemma.\\

 Similarly, for the second equation \eqref{eq:alphabin_sigma}, we begin with replacing the expression for $\envdx{\Lm_{\star}}{W_{\star}}{1}$ to see that
 \begin{align}
  \Exp\left[\left(\envdx{\Lm_{\star}}{W_{\star}}{1} \right)^2 \right]\nn &= \frac{(\la_{\star}\delta-1)^2}{\delta^2(\eta-\Ic(W_{\star}))^2}\left(\eta^2\,\E\left[W_{\star}^2\right] + \Ic\left(W_{\star}\right) 
 + 2\eta\,\E\left[W_{\star}\cdot\ksi_{{W_{\star}}}(W_{\star})\right]\right)\nn \\[5pt]
  &=  \frac{(\la_{\star}\delta-1)^2}{\delta^2(\eta-\Ic(W_{\star}))^2}\left(\eta^2\,(\sig_{\star}^2+1)+ \Ic\left(W_{\star}\right) 
 - 2\eta\right).\label{eq:third_opt_checking}
 \end{align}
 After replacing $\eta$, we can simplify \eqref{eq:third_opt_checking} to reach the following
 \begin{align*}
\Exp\left[\left(\envdx{\Lm_{\star}}{W_{\star}}{1} \right)^2 \right]&= \frac{1-\sig_{\star}^2(1-\sig_{\star}^2\,\Ic(W_{\star}))}{\delta^2(\sig_{\star}^2\,\Ic(W_{\star})+\Ic(W_{\star})-1)}-\frac{2\la_{\star}\,\sig_{\star}^2}{\delta} + \la_{\star}^2(1+\sig_{\star}^{2})\\[5pt]
&= \frac{\Phi(\sig_{\star},\la_{\star})\cdot\sig_{\star}^2}{\delta}= \frac{\sig_{\star}^2}{\delta},
 \end{align*}
 where the last two steps follow from the definition of $\sig_{\star}$ in \eqref{eq:sigopt_thm} and $\Phi(\cdot,\cdot)$ in \eqref{eq:phi_bin}. 
 
 For the third Equation \eqref{eq:lambdabin_sigma} we deduce in a similar way that
 \begin{align*}
 \E\Big[ \ksi_{_{W_{\star}}}(W_{\star})\cdot \envdx{\Lm_{\star}}{W_{\star}}{1}  \Big] &= -\frac{\eta(\la_{\star}\delta-1)}{\delta(\eta-\Ic(W_{\star}))}\E\left[W_{\star}\cdot\ksi_{{W_{\star}}}(W_{\star})\right] - \frac{\la_{\star}\delta-1}{\delta(\eta-\Ic(W_{\star}))} \Ic(W_{\star}) \\&= \la_{\star} - \frac{1}{\delta},
\end{align*}
confirming the RHS of Equation \eqref{eq:lambdabin_sigma}. This completes the proof.

\subsection{Proof of Lemma \ref{cor:LS_bin}}
Let $\ell_2(t) = (1-t)^2$ for $t\in\R$. Using the Equations in \eqref{eq:bin_main} and replacing $\env{\ell_2}{x}{\tau} = \frac{(x-1)^2}{2\tau+1}$ we can solve the equations to find the closed-form formulas for $(\mu,\alpha,\tau)$ for a fixed $\la\ge0$. For compactness, define $F(\cdot,\cdot):\R_{>0}\times\R_{>0}\rightarrow \R_{>0}$ where $F(\delta,\la) := \la\delta + \sqrt{8\la\delta+(\delta(\la+2)-2)^2}$. We derive the following for $\mu_{{\ell_2,\la}}$ and $\alpha_{{\ell_2,\la}}$ and for all $\delta>0$,
\bea
\mu_{{\ell_2,\la}}&= \frac{4\delta\,\E[Z^2]}{2+2\delta+F(\delta,\la)},\nn\\[5pt]
\alpha_{{\ell_2,\la}}^2 &= \frac{\delta \left(2-2\delta-2\la\delta+F(\delta,\la)\right)^2 (2+2\delta+F(\delta,\la))\left(1-\frac{8\delta\left(\E[Z^2]\right)^2(2+F(\delta,\la))}{(2+2\delta+F(\delta,\la))^2}\right)}{2\left(2-2\delta+F(\delta,\la)\right)^2\left(F(\delta,\la)-\la\delta\right)}\nn.
\eea
Using these, we reach $\sig_{\,{\ell_2,\la}}^2 = \alpha_{{\ell_2,\la}}^2/\mu_{{\ell_2,\la}}^2$ as stated in \eqref{eq:sigmaLSreg}. By minimizing $\sig_{\,{\ell_2,\la}}^2$ with respect to $\la\ge0$ we derive $\la_{\rm opt}$ and the resulting $\sig_{\,{\ell_2,\la_{\rm opt}}}^2$ in the statement of the lemma. 

\subsection {Proof of Corollary \ref{cor:lowerbound_binary}}\label{sec:proofofcor_bin}
The proof is analogous to the proof of Corollary \ref{cor:lowerbound}. Here again we use Stam's inequality in Proposition \ref{propo:Fisher} to provide a bound for $\Ic(W_\sig) = \Ic(\sig G+Sf(S))$ based on $\Ic(\sig G) = \sig^{-2}$ and $\Ic(Sf(S))$. First we define 
\bea\label{eq:sigmahat_cor}
\widehat{\sig} := \min_{x\ge0} \left\{s\ge0 : \frac{1}{\delta} + \frac{1}{\delta s^2(\Ic(Sf(S))-1)}-2x+\delta x^2(1+ s ^{-2}) \le 1 \right\}. 
\eea
Next we use Stam's inequality to deduce that :
\bea\nn
\Ic(W_\sig):=\Ic(\sig G +Sf(S) )\le \frac{\Ic(Sf(S))}{1+\sig^2\Ic(Sf(S))}.
\eea
We can use this inequality in the constraint condition of $\sig_{\star}$ in \eqref{eq:sigopt_thm} to deduce that:
\bea
\frac{1}{\delta} + \frac{1}{\delta\sig_{\star}^2(\Ic(Sf(S))-1)}-2\la_{\star}+\delta \la_{\star}^2(1+\sig_{\star}^{-2}) \le 1,
\eea
Thus we find that $(\sig,x)= (\sig_{\star},\la_\star)$ is a feasible solution of the constraint in \eqref{eq:sigmahat_cor}, resulting in :
\bea
\sig_\star \ge \widehat{\sig}.
\eea
To complete the proof of the theorem, we need to find the closed-form $\widehat{\sig}$. Proceeding from \eqref{eq:sigmahat_cor} we derive the following 
\begin{align*}
\widehat{\sig}^2  &= \min_{x\ge0} \left\{s^2 : \frac{1}{s^2}\left(\frac{1}{\delta(\Ic(Sf(S))-1)}+x^2\delta\right)   \le 1+2x-\frac{1}{\delta}-x^2\delta \right\}  \\[5pt]
&=   \min_{x\ge0} \left\{s^2 : \frac{1}{s^2} \le \frac{1+2x-1/\delta-x^2\delta}{\frac{1}{\delta(\Ic(Sf(S))-1)}+x^2\delta} \right\} \\[5pt]
&= \left(\max_{x\ge0} \left\{\frac{1+2x- 1/\delta-x^2\delta}{\frac{1}{\delta(\Ic(Sf(S))-1)}+x^2\delta}\right\}\right)^{-1}.
\end{align*}
The first line follows by algebraic simplifications in \eqref{eq:sigmahat_cor}. The second line is true since by Cramer-Rao bound (see Proposition \ref{propo:Fisher} (d)) $\Ic(Sf(S)) \ge (Var[Sf(S)])^{-1}$; thus $\Ic(Sf(S))\ge1$. Noting that the right hand-side of the inequality is independent of $\sig$ and can take positive values for some $x\ge0$ we conclude the last line. Optimizing with respect to the non-negative variable $x$ in the last line completes the proof and yields the desired result in the statement of the corollary. 

\section{Comparison to a Simple Averaging Estimator}\label{sec:averaging}
In this section, we compare the performance of optimally ridge-regularized ERM to the following simple  averaging estimator
\bea\label{eq:av_est}
\wh_{\rm ave} = \frac{1}{m}\sum_{i=1}^{m} y_i\ab_i.
\eea

This estimator  is closely related to the family of RERM estimators studied in this paper. To see this, note that $\wh_{\rm ave}$ can be expressed as the solution to ridge-regularized ERM with $\la =1$ and linear loss function $\Lm (x) = -x$ for all $x \in \R$:
\begin{align*}
\wh_{\rm ave} = \arg \min_{\w \in \R ^n}\; \frac{1}{2m} \sum_{i=1}^m \left\|y_i \ab_i -\w\right\|_2^2 = \arg \min_{\w \in \R ^n} \;\frac{1}{m}\sum_{i=1}^m -y_i \ab_i^T\w + \frac{1}{2} \|\w\|_2^2.
\end{align*}
Moreover, it is not hard to check that the correlation performance of  $\wh_{\rm ave}$ is the same as that of the solution of RLS with regularization $\la$ approaching infinity. 

It is in fact possible to exploit these relations of the estimator to the RERM family in order to evaluate its asymptotic performance using the machinery of this paper (i.e., by using the Equations \eqref{eq:bin_sys}). However, a more direct evaluation that uses the closed form expression in \eqref{eq:av_est} is preferable here.  In fact, it  can be easily checked that the following limit is true in the high-dimensional asymptotic regime:
\bea
\forall \delta>0 \; : \quad\quad\quad \corr{\wh_{\rm ave}}{\x_0} \rP\frac{1}{1+\frac{1}{\delta} \nu_f^{-2}},\label{eq:aveper}
\eea
where recall our notation $\nu_f=\E[S f(S)],~S\sim\Nn(0,1)$. The use of the simple averaging estimator for signal recovery in generalized linear models (also, in single-index models) has been previously investigated for example in \cite{lu2017phase}.

A favorable feature of $\wh_{\rm ave}$ is its computational efficiency. In what follows, we use our lower bounds on the performance of general RERM estimators, to evaluate its suboptimality gap compared to more complicated alternatives. To begin, in view of \eqref{eq:aveper} and  \eqref{eq:corr_lim} let us define the corresponding ``effective error parameter" 
\bea
\sigma^2_{\rm ave} =\frac{1}{\delta} \nu_f^{-2}. \label{eq:ave}
\eea
First, we compare this value with the error of regularized LS. Let $\wh_{\rm LS}$ be the solution to \emph{unregularized} LS for $n>m$. It can be checked (e.g., \cite{NIPS}) that
\bea
\corr{\wh_{\rm LS}}{\x_0} \rP\frac{1}{1+ \sig_{{\rm LS}}^2 },\quad\text{where}~~\sig_{{\rm LS}}^2 := \frac{1}{\delta-1} \big(\nu_f^{-2}-1\big).
\eea
Directly comparing this to \eqref{eq:ave}, we find that 
$
\frac{\sigma^2_{\rm LS}}{\sigma^2_{\rm ave}} 
= \Big(\frac{1}{1-1/\delta}\Big)\,\big( 1 - \nu_f^2 \big),
$
for all $\delta>1$.  In other words,
\bea
\sigma^2_{\rm ave}\gtrless \sigma^2_{\rm LS}   \quad\Longleftrightarrow\quad  \delta \gtrless \nu_f^{-2}.
\eea
Next, we study the performance gap of the averaging estimator from the optimal RERM. For this, we use Corollary \ref{cor:lowerbound_binary} to compare $\sigma^2_{\rm ave}$ to the lower bound $\sig_\star$. We find that for any $\delta>0$ and any link function $f$ satisfying the assumptions of Corollary \ref{cor:lowerbound_binary}:
\bea
1\;\ge\;\frac{\sig_\star^2}{\sig_{\rm ave}^2}\; \ge \;\delta\,\nu_f^2\cdot H_\delta\Big(\Ic(Sf(S))\Big).
\eea
We complement these bounds with numerical simulations  in Section \ref{sec:numeric}.


\section{Gains of Regularization}\label{sec:unreg_opt}
\subsection{Linear models}\label{sec:gains_lin}
In this section, we study the impact of the regularization parameter on the best achievable performance. For this purpose, we compare $\alpha_\star$, the best achievable performance of ridge-regularized case, to the best achievable performance among non-regularized empirical risk minimization with convex losses denoted by $\alpha_{\rm ureg}$. By definition of $\alpha_{\rm ureg}$, for all convex losses $\Lm$, in the regime of $\delta>1$ it holds that, $\alpha_{\rm ureg}\;\le\;\alpha_{{\Lm,\,0}}.$ In \cite{bean2013optimal}, the authors compute a tight lower bound on $\alpha_{\rm ureg}$ and show that it is attained provided that $p_Z$ is log-concave. Our next result bounds the ratio $\alpha_{\star}^2\, /\, \alpha_{_{\rm ureg}}^2$, illustrating the impact of regularization for a wide range of choices of $Z\sim \mathcal{D}$ and any $\delta>1$.
\begin{cor}\label{cor:gains_reg}
Let the assumptions of Corollary \ref{cor:lowerbound} hold and $\delta >1$. Then it holds that:
\bea\label{eq:gains_reg}
 \frac{(\delta-1)}{\E[Z^2]} \,h_{{\delta}} \left(\frac{1}{\Ic(Z)}\right)\;\le\;\frac{\alpha_{\star}^2}{\alpha_{\,{\rm ureg}}^2}\;\le\; \min\Big\{(\delta-1)\,\Ic(Z),\,1 \Big\}.
 \eea
\end{cor}
\begin{proof}
In order to obtain an upper bound for $\alpha_\star^2/\alpha_{\rm ureg}^2$ first we find a lower bound for $\alpha_{\rm ureg}^2$. We have $$\alpha_{{\rm ureg}}^2\,\Ic(V_{\alpha_{_{\rm ureg}}}) = \frac{1}{\delta},$$ thus we may apply the Stam's inequality (as stated in Proposition \ref{propo:Fisher}(f)) for $\Ic(V_{\alpha_{_{\rm ureg}}}) $ to derive the following lower bound :
\bea\label{eq:unreg_linear_lower}
\alpha_{{\rm ureg}}^2 \ge \frac{1}{(\delta-1)\Ic(Z)}.
\eea
Also note that it holds that $\alpha_\star^2\le\alpha_{{\ell_2,\la_{\rm opt}}}^2$. Thus by recalling Lemma \ref{cor:LS_reg} and the fact that the function $h_{\delta}(\cdot)\le1$ for all $\delta\ge0$ we deduce that $\alpha_\star^2\le1$. Additionally since $\alpha_\star^2 \le \alpha_{\rm ureg}^2$, we conclude the upper bound in the statement of the Corollary.
To proceed, we use the Cramer-Rao bound (see Proposition \ref{propo:Fisher}(d)) for $\Ic(V_{\alpha_{{\rm ureg}}}) $ to derive the following upper bound for $\alpha_{_{\rm ureg}}^2$ which holds for all $\delta>1$: 
\bea\nn
\alpha_{{\rm ureg}}^2 \le \frac{\E[Z^2]}{\delta-1}.
\eea
This combined with the result of Corollary \ref{cor:lowerbound} derives the lower bound in the statement of the corollary and completes the proof.
\end{proof}
Importantly, based on \eqref{eq:gains_reg} we find that as $\delta\rightarrow 1$ the ratio $\alpha_{\star}^2\, /\, \alpha_{{\,\rm ureg}}^2$ reaches zero, implying the large gap between $\alpha^\star$ and $\alpha_{{\,\rm ureg}}$ in this regime. 
In the highly under-parameterized regime where $\delta\rightarrow \infty$, by computing the limit in the lower bound our bound gives
\bea
\frac{1}{\E[Z^2]\,\Ic(Z)}\;\le\;\lim_{\delta\rightarrow\infty} \frac{\alpha_{\star}^2}{\alpha_{\,{\rm ureg}}^2} \;\le \;1 \,.\label{eq:not_so_good_bound}
\eea
For example, we see that in this regime when $Z$ is close to a Gaussian distribution such that $\Ic(Z) \approx 1/\E[Z^2]$, then provably $\alpha_\star \approx \alpha_{\rm ureg}$, implying that impact of regularization is infinitesimal in the resulting error. We remark that for other distributions that are far from Gaussian in the sense 
$\Ic(Z) \gg 1/\E[Z^2]$ the simple lower bound in \eqref{eq:not_so_good_bound} is not tight; this is because the bound of Corollary \ref{cor:lowerbound} is not tight in this case.

\subsection{Binary models}\label{sec:gains_bin}
In order to demonstrate the impact of regularization on the performance of ERM based inference, we compare $\sig_\star$ with the optimal error of the non-regularized ERM for $\delta>1$ which we denote by $\sig_{\rm ureg}$. Thus $\sig_{\rm ureg}$ satisfies for all convex losses that $\sig_{\rm ureg} \le \sig_{\Lm,0}$. The general approach for determining $\sig_{{\rm ureg}}$ is discussed in \cite{taheri2020sharp} in which the authors also show the achievability of $\sig_{{\rm ureg}}$ for well-known models such as the Signed and Logistic models. 
\par
Our next result quantifies the gap between $\sig_{{\rm ureg}}$ and $\sig_\star$ in terms of the label functions $f$ and $\delta>1$. 
\begin{cor}\label{cor:gains_bin}
Let the assumptions of Theorem \ref{thm:lowerbound_bin} hold and $\delta >1$. Further assume the label function $f$ is such that $p_{_{S\cdot f(S)}}(x)$ is differentiable and positive for all $x \in \R $. Then it holds that:
\bea\label{eq:unregGap_bin}
 \frac{(\delta-1)\nu_f^2}{1-\nu_f^2}\,H_\delta\Big(\Ic(\,Sf(S)\,)\Big)\;\le\;\frac{\sig_{\star}^2}{\sig_{_{\rm ureg}}^2}\;\le\; \min\left\{\frac{\delta-1}{\delta}\cdot\frac{\Ic(Sf(S))-1}{ \nu_f^2}\,,\,1\,\right\}.
\eea
\end{cor}
\begin{proof}
To provide the bounds of the ratio $\sig_{\star}^2/\sig_{_{\rm ureg}}^2$, we follow a similar argument stated in the proof of Corollary \ref{cor:gains_reg}. First, we use the result in \cite{taheri2020sharp} which states that for $\sig_{\rm ureg}^2$ and all $\delta>1$ it holds that
\bea
\sig_{\rm ureg}^2 \ge \frac{1}{(\delta-1)(\Ic(Sf(S))-1)}.
\eea 
Since it trivially holds that $\sig_{\star}^2 \le \sig_{\ell_2,\la_{\rm opt}}^2$ and also by noting that $ \sig_{\ell_2,\la_{\rm opt}}^2$ as derived by Lemma \ref{cor:LS_bin} satisfies $\sig_{\ell_2,\la_{\rm opt}}^2 \le \frac{1}{\delta \nu_f^2}$ for all $\delta>0$ (which is followed by the fact that $H_\delta(x)\le \frac{x}{(x-1)\delta}$), we conclude that 
\bea
\sig_\star^2\le \frac{1}{\delta \nu_f^2}. 
\eea
Additionally since it trivially holds that $\sig_\star^2\le\sig_{\rm ureg}^2$ we conclude the upper bound in the statement of the corollary.
We proceed with proving the lower bound in the statement of the corollary. For this purpose, first we derive an upper bound for $\sig_{\rm ureg}^2$.  Using the fact that $\sig_{_{\rm ureg}}^2$ satisfies :
\bea
\frac{1-\sig_{_{\rm ureg}}^2(1-\sig_{_{\rm ureg}}^2\Ic(W_{\rm ureg}))}{\delta \sig_{_{\rm ureg}}^2(\sig_{_{\rm ureg}}^2\Ic(W_{\rm ureg})+\Ic(W_{\rm ureg})-1)}= 1
\eea
as well as the Cramer-Rao lower bound (Proposition \ref{propo:Fisher}(d)) for $\Ic(W_{\rm ureg})$ we may deduce that :
\bea
\sig_{{\rm ureg}}^2 \le  \frac{(\delta-1)\nu_f^2}{1-\nu_f^2}.
\eea
This combined with the lower bound on $\sig_\star^2$ as stated in Corollary \ref{cor:lowerbound_binary} proves the lower bound in the statement of the corollary and completes the proof.
\end{proof}
Importantly, as shown by \eqref{eq:unregGap_bin}, in the case of $\delta$ being close to 1, one can see that both of the bounds in \eqref{eq:unregGap_bin} vanish. This shows the large gap between $\sig_{{\rm ureg}}$ and $\sig_{\star}$ and further implies the benefit of regularization in this regime. 
When $\delta\rightarrow\infty$ i.e. in the highly under-parameterized regime, by deriving the limits as well as using Proposition \ref{propo:Fisher} (d), we see that \eqref{eq:unregGap_bin} yields: 
\bea\label{eq:unregGap_bin_infty}
\frac{\nu_f^2}{1-\nu_f^2}\cdot\frac{1}{\Ic(Sf(S))-1}\le \lim_{\delta\rightarrow\infty}\; \frac{\sig_{\star}^2}{\sig_{{\rm ureg}}^2} \le 1.
\eea
Thus in this case both the values of $\sig_\star$ and $\sig_{\rm ureg}$ are approaching zero with the ratio depending on the properties of $Sf(S)$. For models such as Logistic with small signal strength (i.e. small $\|\x_0\|$) where $\Ic(Sf(S)) \approx 1/(1-\nu_f^2)$, one can derive that based on \eqref{eq:unregGap_bin_infty} the ratio reaches 1, which confirms the intuition that for large values of $\delta$ the impact of regularization is almost negligible.


\section{Additional Experiments}\label{sec:numeric}
In this section, we present additional numerical results comparing the bounds of Theorems \ref{thm:lowerbound_reg} and \ref{thm:lowerbound_bin} to the performance of the following: (i) Ridge-regularized Least-Squares (RLS); (ii) optimal unregularized ERM (Section \ref{sec:unreg_opt}); (iii) a simple averaging estimator (see Section \ref{sec:averaging}). Figure \ref{fig:fig_app}(Top Left) plots the asymptotic squared error $\alpha^2$ of these estimators for linear measurements with $Z\sim\texttt{Laplace}(0,2)$. 
Similarly, Figure \ref{fig:fig_app}(Top Right) and Figure \ref{fig:fig_app}(Bottom) plot the effective error term $\sigma$ for Logistic data with $\|\x_0\|_2=1$, and the limiting value $\rho$ of the correlation measure for Logistic data with $\|\x_0\|_2=10$, respectively.
 The red squares represent the performance of optimally tuned ERM (as per Lemmas \ref{thm:opt_reg} and \ref{thm:opt_bin}) derived numerically by running GD, as previously described in the context of Figure \ref{fig:fig}. 

\begin{figure}
\centering
\begin{subfigure}{.47\textwidth}
  \centering
  \includegraphics[width=0.85\linewidth,height= 5.5cm]{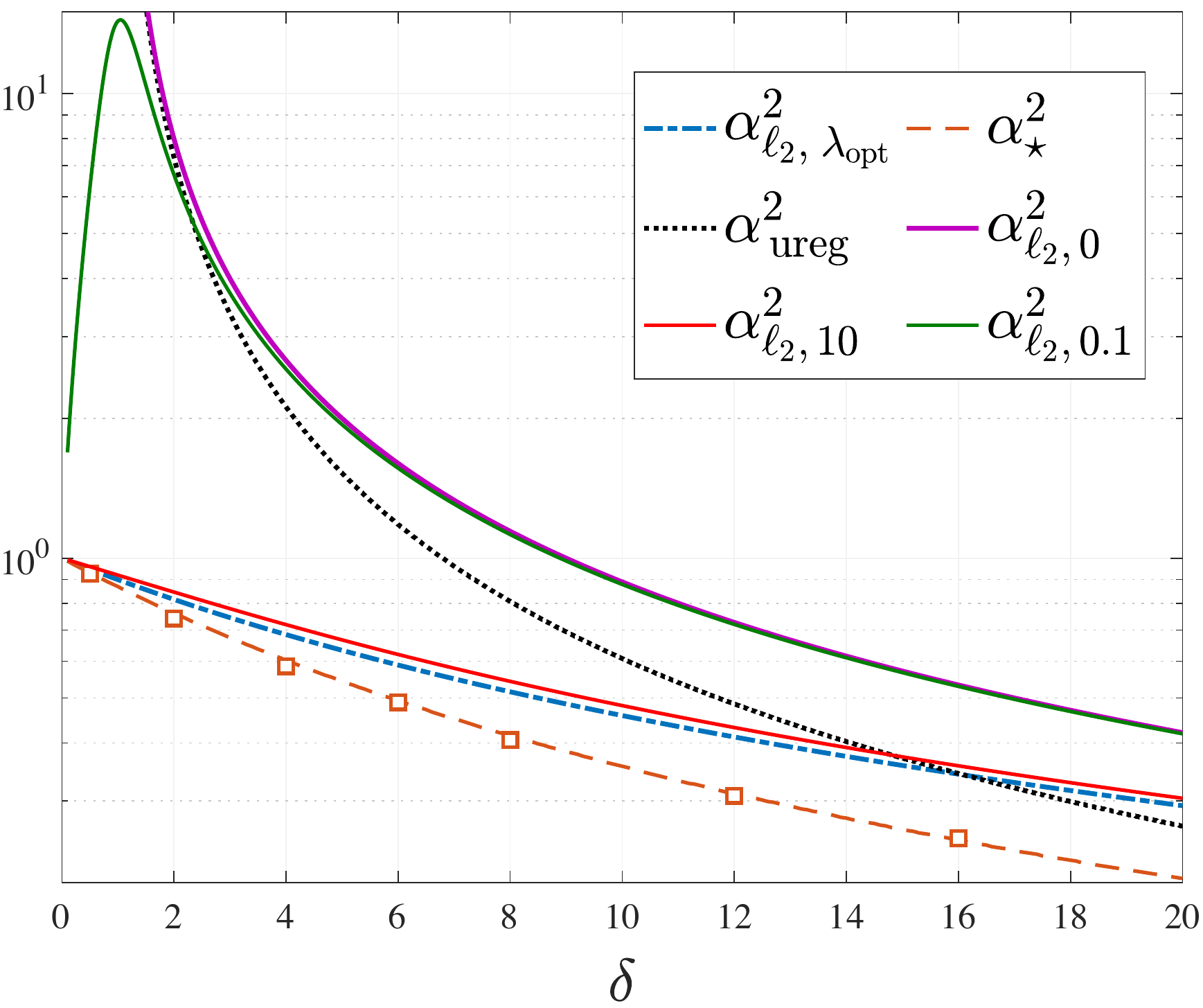}
\end{subfigure}
\begin{subfigure}{0.47\textwidth}
  \centering
  \includegraphics[width=0.85\linewidth,height= 5.5cm]{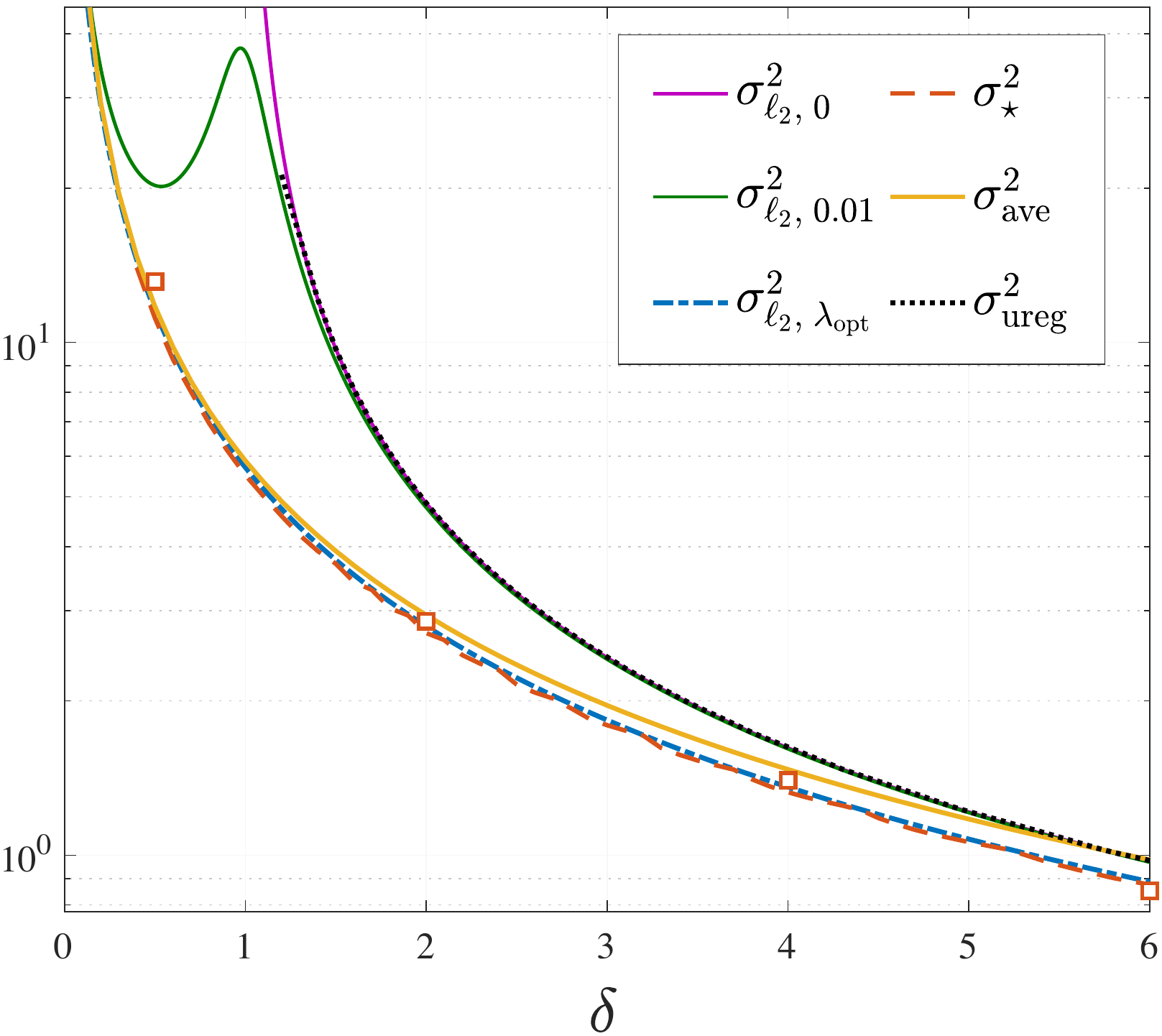}
\end{subfigure}\\
\vspace{.2in}
\begin{subfigure}{0.45\textwidth}
  \centering
  \includegraphics[width=0.9\linewidth,height= 5.7cm]{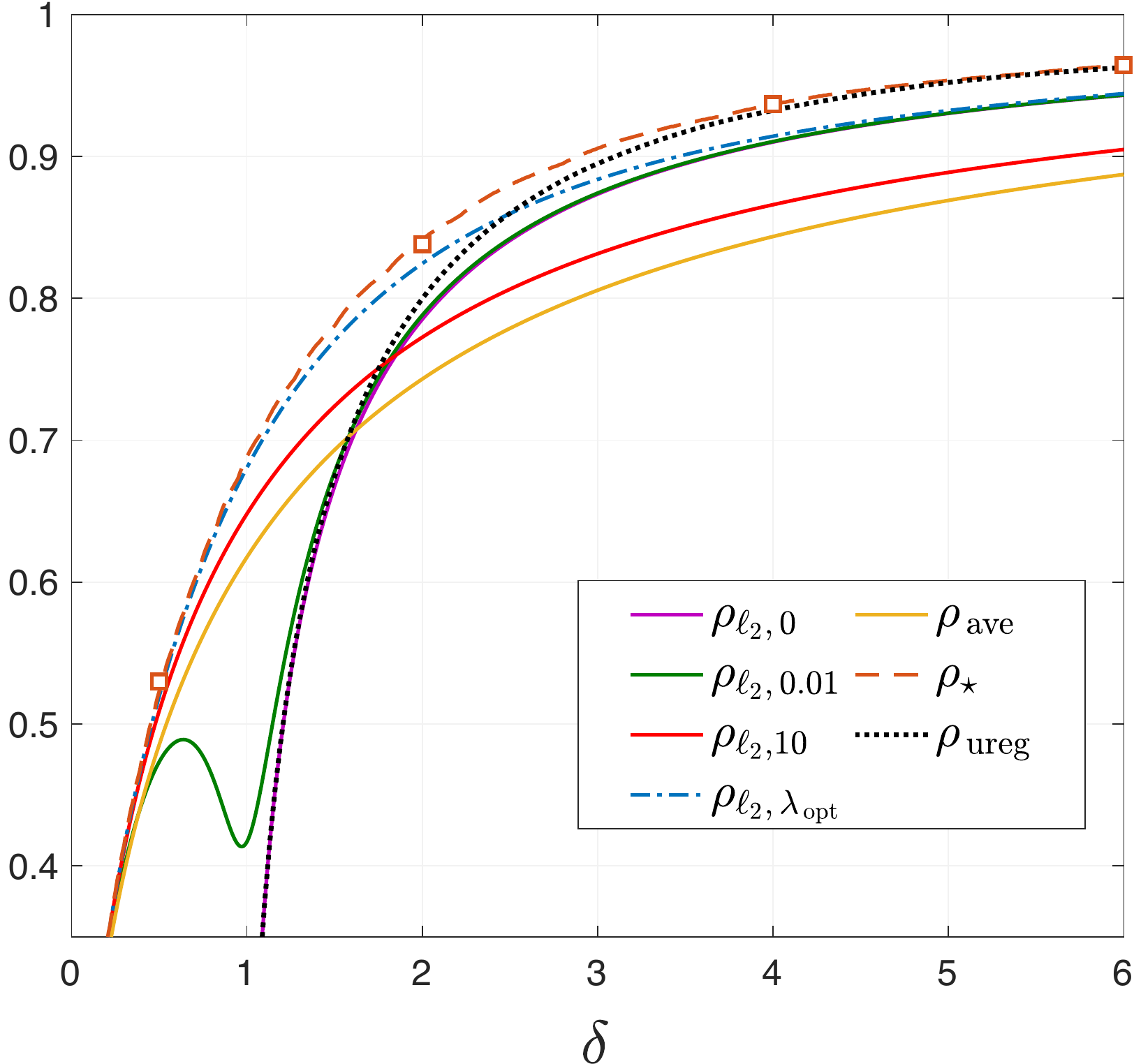}
\end{subfigure}
\caption{Fundamental error bounds derived in this paper compared to RLS, averaging estimator and optimal unregularized ERM for: (Top Left) a linear model with $Z\sim\texttt{Laplace}(0,2)$, (Top Right) a binary Logistic model with $\|\x_0\|_2=1$ , (Bottom) a binary Logistic model with $\|\x_0\|_2=10$ (here shown is correlation measure \eqref{eq:corr_lim}). The red squares correspond to numerical evaluation of the performance of the optimally tuned RERM as derived in Lemmas \ref{thm:opt_reg} and \ref{thm:opt_bin}.}
\label{fig:fig_app}
\end{figure}
The numerical findings in Figures \ref{fig:fig} and \ref{fig:fig_app} validate the theoretical findings of Sections \ref{sec:LS_linear} and \ref{LS_binary}, regarding sub-optimality of RLS for Laplace noise and Logistic binary model (with large $\|\x_0\|$) and optimality of $\la$-tuned RLS for Logistic model with small $\|\x_0\|$. Furthermore, by comparing the optimal performance of unregularized ERM to the optimal errors of RERM in both Figures \ref{fig:fig} and \ref{fig:fig_app}, we confirm the the theoretical guarantees of Section \ref{sec:unreg_opt} regarding the impact of regularization in the regime of small $\delta$ for both linear and binary models.

\begin{figure}
\centering
\begin{subfigure}{.46\textwidth}
  \centering
  \includegraphics[width=0.89\linewidth,height= 4.8cm]{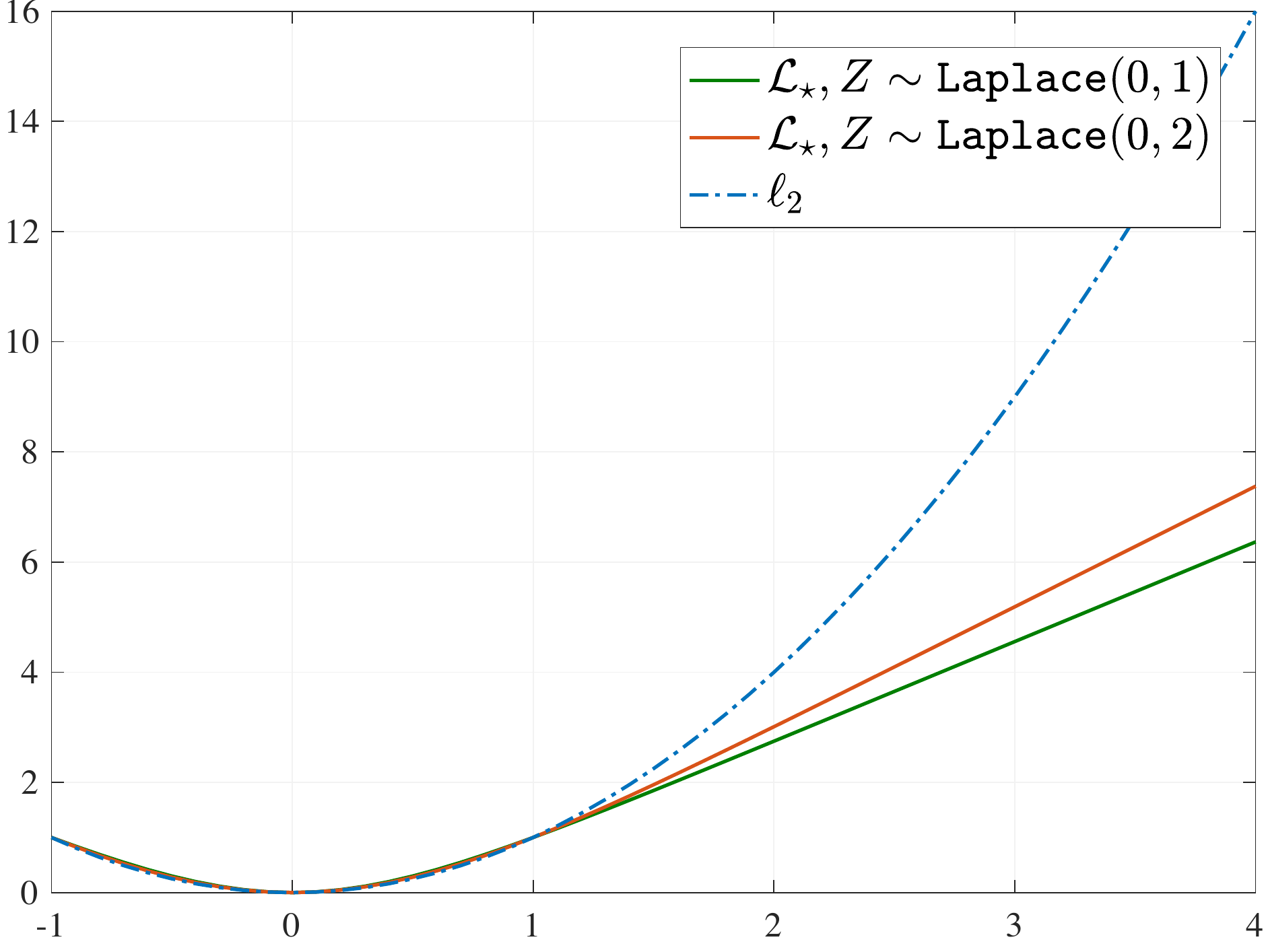}
  \label{fig:lopt_lin}
\end{subfigure}
\begin{subfigure}{0.46\textwidth}
  \centering
  \includegraphics[width=0.89\linewidth,height=4.8cm]{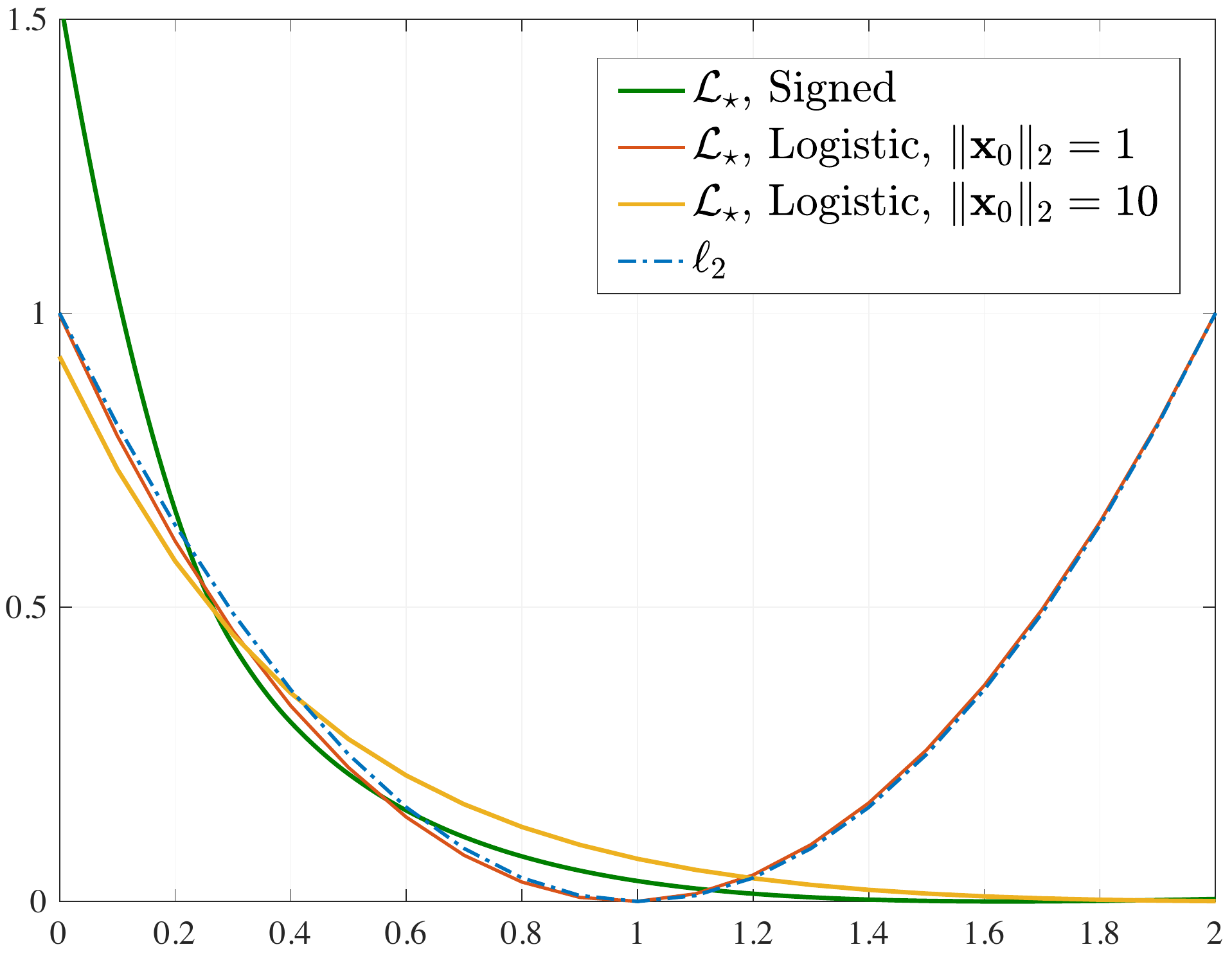}
  \label{fig:lopt_bin}
\end{subfigure}%
\caption{Illustrations of the proposed loss functions achieving optimal performance (as in Lemmas \ref{thm:opt_reg} and \ref{thm:opt_bin}), for three special cases: a linear model with additive Laplace noise, the binary logistic model and the binary signed model. Here, in both plots, we fix $\delta=2$. The curves are appropriately shifted and rescaled to allow direct comparison  to the least-squares loss function.}
\label{fig:lopt}
\end{figure}

\subsection{Optimal Tuning in Special Cases}\label{sec:optimaltuning_app}
Figure \ref{fig:lopt} depicts the candidate for optimal loss function derived in Lemmas \ref{thm:opt_reg} and \ref{thm:opt_bin}, for specific linear and binary models discussed in this paper. To allow for a direct comparison with the least-squares loss function, the optimal losses for the linear models are shifted such that $\Lm_\star\ge0$ and rescaled such that $\Lm_\star(1)=1$. Similarly, for the Logistic model with $\|\x_0\|=1$, the optimal loss is rescaled such that $\Lm_\star(1)=0$ and $\Lm_\star(2)=1$. Interestingly, for this model, $\Lm_\star$, when rescaled (which results in no change in performance by appropriately rescaling $\la_\star$) is similar to the least-squares loss. This confirms the (approximate) optimality of optimally-tuned RLS for this model and further verifies the numerical observations in Figure \ref{fig:fig_app} (Top Right) and the theoretical guarantees of Section \ref{LS_binary} for this model.


\end{document}